\documentclass[10pt,onecolumn]{saim-paper}

\newcommand{\ProjectTitle}{Corrected Block Sparse Spectral Optimization: Exact Models, Finite Identification, and Local Dynamics}
\newcommand{\ProjectShortTitle}{Corrected Block Sparse Optimization}
\newcommand{\ProjectSubtitle}{A theory-first reconstruction of block threshold-and-polar methods}

\newcommand{\ProjectDate}{October 2026}
\newcommand{\ProjectVersion}{v0.5}
\newcommand{\ProjectKeywords}{FACET, sparse PCA, Stiefel manifold, support identification, KŁ convergence, local spectral dynamics}

\makeatletter
\newcounter{AuthorCounter}
\newcounter{AffiliationCounter}

\newcommand{\ProjectMetaSep}{\enspace\textbullet\enspace}

\newcommand{\AuthorEntry}[2]{#1\textsuperscript{#2}}
\newcommand{\AffiliationEntry}[2]{\textsuperscript{#1}#2}

\newcommand{\AuthorListTypeset}{}
\newcommand{\AuthorListPlain}{}
\newcommand{\AuthorListPDF}{}
\newcommand{\AuthorListBib}{}
\newcommand{\AffiliationListTypeset}{}
\newcommand{\AffiliationListPlain}{}
\newcommand{\AffiliationListPDF}{}
\newcommand{\EmailListTypeset}{}
\newcommand{\EmailListPlain}{}
\newcommand{\EmailListPDF}{}
\newcommand{\AuthorEmailListTypeset}{}
\newcommand{\AuthorEmailListPlain}{}
\newcommand{\AuthorEmailListPDF}{}
\newcommand{\LeadAuthor}{}
\newcommand{\LeadAuthorEmail}{}
\newcommand{\CorrespondingAuthorName}{}
\newcommand{\CorrespondenceEmail}{}

\newcommand{\ResetPeople}{%
  \setcounter{AuthorCounter}{0}%
  \setcounter{AffiliationCounter}{0}%
  \gdef\AuthorListTypeset{}%
  \gdef\AuthorListPlain{}%
  \gdef\AuthorListPDF{}%
  \gdef\AuthorListBib{}%
  \gdef\AffiliationListTypeset{}%
  \gdef\AffiliationListPlain{}%
  \gdef\AffiliationListPDF{}%
  \gdef\EmailListTypeset{}%
  \gdef\EmailListPlain{}%
  \gdef\EmailListPDF{}%
  \gdef\AuthorEmailListTypeset{}%
  \gdef\AuthorEmailListPlain{}%
  \gdef\AuthorEmailListPDF{}%
  \gdef\LeadAuthor{}%
  \gdef\LeadAuthorEmail{}%
  \gdef\CorrespondingAuthorName{}%
  \gdef\CorrespondenceEmail{}%
}

\newcommand{\Affiliation}[2]{%
  \stepcounter{AffiliationCounter}%
  \expandafter\xdef\csname AffiliationMark@#1\endcsname{\arabic{AffiliationCounter}}%
  \expandafter\gdef\csname AffiliationText@#1\endcsname{#2}%
  \ifnum\value{AffiliationCounter}>1\relax
    \g@addto@macro\AffiliationListTypeset{\ProjectMetaSep}%
    \g@addto@macro\AffiliationListPlain{\ProjectMetaSep}%
    \g@addto@macro\AffiliationListPDF{; }%
  \fi
  \g@addto@macro\AffiliationListTypeset{\AffiliationEntry{\csname AffiliationMark@#1\endcsname}{#2}}%
  \g@addto@macro\AffiliationListPlain{#2}%
  \g@addto@macro\AffiliationListPDF{\csname AffiliationMark@#1\endcsname: #2}%
}

\ExplSyntaxOn
\cs_new_protected:Npn \saim_unknown_affiliation:n #1
  {
    \PackageError{saim-metadata}{Unknown affiliation key `#1'}
      {Declare it first with \string\Affiliation{#1}{Institution name}.}
  }
\cs_new:Npn \saim_affiliation_mark:n #1
  {
    \cs_if_exist_use:cF { AffiliationMark@\tl_trim_spaces:n {#1} }
      { \saim_unknown_affiliation:n {#1} }
  }
\int_new:N \l_saim_affiliation_mark_index_int
\cs_new:Npn \saim_affiliation_marks:n #1
  {
    \int_zero:N \l_saim_affiliation_mark_index_int
    \clist_map_inline:nn {#1}
      {
        \int_incr:N \l_saim_affiliation_mark_index_int
        \int_compare:nNnT {\l_saim_affiliation_mark_index_int} > {1} {,}%
        \saim_affiliation_mark:n {##1}
      }
  }
\NewDocumentCommand{\AffiliationMarks}{m}{\saim_affiliation_marks:n {#1}}
\ExplSyntaxOff

\newcommand{\Author}[3]{%
  \stepcounter{AuthorCounter}%
  \ifnum\value{AuthorCounter}=1\relax
    \gdef\LeadAuthor{#1}%
    \gdef\LeadAuthorEmail{#2}%
    \gdef\CorrespondingAuthorName{#1}%
    \gdef\CorrespondenceEmail{#2}%
  \else
    \g@addto@macro\AuthorListTypeset{\ProjectMetaSep}%
    \g@addto@macro\AuthorListPlain{\ProjectMetaSep}%
    \g@addto@macro\AuthorListPDF{; }%
    \g@addto@macro\AuthorListBib{\ and }%
  \fi
  \g@addto@macro\AuthorListTypeset{\AuthorEntry{#1}{\AffiliationMarks{#3}}}%
  \g@addto@macro\AuthorListPlain{#1}%
  \g@addto@macro\AuthorListPDF{#1}%
  \g@addto@macro\AuthorListBib{#1}%
  \if\relax\detokenize{#2}\relax\else
    \ifx\EmailListPlain\@empty\else
      \g@addto@macro\EmailListTypeset{\ProjectMetaSep}%
      \g@addto@macro\EmailListPlain{\ProjectMetaSep}%
      \g@addto@macro\EmailListPDF{; }%
      \g@addto@macro\AuthorEmailListTypeset{\ProjectMetaSep}%
      \g@addto@macro\AuthorEmailListPlain{\ProjectMetaSep}%
      \g@addto@macro\AuthorEmailListPDF{; }%
    \fi
    \g@addto@macro\EmailListTypeset{\href{mailto:#2}{\nolinkurl{#2}}}%
    \g@addto@macro\EmailListPlain{#2}%
    \g@addto@macro\EmailListPDF{#2}%
    \g@addto@macro\AuthorEmailListTypeset{#1\enspace\href{mailto:#2}{\texttt{<#2>}}}%
    \g@addto@macro\AuthorEmailListPlain{#1\enspace<#2>}%
    \g@addto@macro\AuthorEmailListPDF{#1 <#2>}%
  \fi
}

\makeatother

\newcommand{\ProjectAuthorsDisplay}{\texorpdfstring{\AuthorListTypeset}{\AuthorListPDF}}

\newcommand{\ProjectAuthorsPDF}{\AuthorListPDF}

\newcommand{\ProjectAffiliationsDisplay}{\texorpdfstring{\AffiliationListTypeset}{\AffiliationListPDF}}

\newcommand{\ProjectEmailsDisplay}{\texorpdfstring{\EmailListTypeset}{\EmailListPDF}}

\newcommand{\ProjectContactDisplay}{\ProjectEmailsDisplay}

\ResetPeople
\Affiliation{saim}{sAIm Labs}
\Author{Omid Shams Solari}{omid.solari@saimai.com, solari@berkeley.edu}{saim}

\newcommand{\DeclareProjectPaperPeople}{%
  \gdef\authorlist{\ProjectAuthorsDisplay}%
  \gdef\affiliationlist{\ProjectAffiliationsDisplay}%
}

\makeatletter
\newcommand{\ProjectMaybeAddEmailMetadata}{%
  \ifx\EmailListPlain\@empty\else
    \metadata[Emails]{\ProjectContactDisplay}%
  \fi
}
\makeatother

\graphicspath{{shared/figures/}{shared/brand/logos/}}

\usepackage{amsmath,amsfonts,amssymb,mathtools}
\usepackage{amsthm}
\usepackage{algorithm}
\usepackage{algpseudocode}
\usepackage{tikz}
\usepackage{longtable}
\usepackage{multirow}
\usepackage{enumitem}
\usepackage{bm}
\usetikzlibrary{positioning,arrows.meta,calc,fit,backgrounds,shapes.geometric,decorations.pathreplacing,patterns}
\newcolumntype{Y}{>{\raggedright\arraybackslash}X}
\newcolumntype{Z}{>{\centering\arraybackslash}X}
\newcolumntype{C}[1]{>{\centering\arraybackslash}p{#1}}
\newcolumntype{L}[1]{>{\raggedright\arraybackslash}p{#1}}

\providecommand{\R}{\mathbb{R}}

\providecommand{\norm}[1]{\left\lVert #1 \right\rVert}
\providecommand{\abs}[1]{\left\lvert #1 \right\rvert}

\providecommand{\ip}[2]{\left\langle #1, #2 \right\rangle}

\tikzset{
  arrow/.style={-{Stealth[length=2.5mm]}, thick, saimdark},
  darrow/.style={-{Stealth[length=2.5mm]}, thick, dashed, saimdark},
  csblock/.style={rectangle, rounded corners=2pt, thick, fill=saimbg, text=saimdark,
                  minimum height=1.0cm, minimum width=3.6cm, align=center,
                  font=\small\rmfamily},
  csfinal/.style={rectangle, rounded corners=2pt, thick, fill=saimindigo, text=white,
                  minimum height=1.0cm, minimum width=3.6cm, align=center,
                  font=\small\rmfamily},
  cslevel1/.style={rectangle, rounded corners=2pt, fill=saimindigo, text=white,
                   text width=9em, text centered, minimum height=2.2em,
                   font=\small\rmfamily\bfseries},
  cslevel2/.style={rectangle, rounded corners=2pt, fill=saimbg, text=saimdark,
                   text width=7.5em, text centered, minimum height=1.8em,
                   font=\scriptsize\rmfamily},
  cscrosscut/.style={rectangle, rounded corners=2pt, fill=saimgray, text=saimdark,
                     draw=saimblue, dashed, text width=8em, text centered,
                     minimum height=1.8em, font=\scriptsize\rmfamily},
  csmethodbox/.style={rectangle, fill=saimindigo, text=white, minimum width=2.6cm,
                      minimum height=0.8cm, font=\small\rmfamily\bfseries,
                      text centered},
  cspropbox/.style={rectangle, fill=saimbg, text=saimdark, minimum width=2.6cm,
                    minimum height=0.6cm, font=\scriptsize\rmfamily,
                    text centered, text width=2.6cm},
  csstagelbl/.style={font=\rmfamily\bfseries\small, text=saimindigo},
  boxbase/.style={draw,rounded corners=2pt,minimum height=9mm,align=center,
                  line width=0.6pt,font=\small\rmfamily,inner sep=5pt},
  primarybox/.style={boxbase,fill=PrimaryPale,draw=Primary},
  secondarybox/.style={boxbase,fill=SecondaryPale,draw=Secondary},
  neutralbox/.style={boxbase,fill=black!3,draw=Rule},
  successbox/.style={boxbase,fill=Positive!10,draw=Positive},
  arr/.style={-{Stealth[length=2.2mm]},line width=0.7pt,draw=Ink},
  arrPrimary/.style={-{Stealth[length=2.2mm]},line width=0.9pt,draw=Primary},
  arrSecondary/.style={-{Stealth[length=2.2mm]},line width=0.9pt,draw=Secondary},
  tier/.style={draw=Rule,rounded corners=4pt,dashed,line width=0.5pt,inner sep=3mm}
}

\providecommand{\R}{\mathbb{R}}
\providecommand{\St}{\mathrm{St}}
\providecommand{\supp}{\operatorname{supp}}

\providecommand{\Diag}{\operatorname{Diag}}

\providecommand{\Polar}{\operatorname{Polar}}
\providecommand{\Soft}{\mathcal{S}}

\providecommand{\norm}[1]{\left\lVert #1 \right\rVert}
\providecommand{\ip}[2]{\left\langle #1,#2\right\rangle}

\newtheorem{theorem}{Theorem}[section]
\newtheorem{proposition}[theorem]{Proposition}
\newtheorem{lemma}[theorem]{Lemma}
\newtheorem{corollary}[theorem]{Corollary}
\theoremstyle{definition}

\newtheorem{assumption}[theorem]{Assumption}

\theoremstyle{remark}

\SAIMPaperTheoremStyle{theorem}
\SAIMPaperTheoremStyle{proposition}
\SAIMPaperTheoremStyle{lemma}
\SAIMPaperTheoremStyle{corollary}
\SAIMPaperTheoremStyle{definition}
\SAIMPaperTheoremStyle{assumption}
\SAIMPaperTheoremStyle{conjecture}
\SAIMPaperTheoremStyle{remark}

\paperlabel{sAIm RESEARCH}
\title{\ProjectTitle\thanks{%
Code and reproducibility: library \texttt{facet-core} on PyPI
(\url{https://pypi.org/project/facet-core/}, imported as \texttt{facet});
paper-specific drivers, configs, and smoke tests at
\url{https://github.com/osolari/facet-bgsp};
archived companion release \href{https://doi.org/10.5281/zenodo.23460513}{doi:10.5281/zenodo.23460513}
(concept DOI, resolves to the latest version).%
}}
\shorttitle{\ProjectShortTitle}
\paperversion{\ProjectVersion}
\DeclareProjectPaperPeople
\ProjectMaybeAddEmailMetadata
\metadata[Updated]{\ProjectDate}
\metadata[Version]{\ProjectVersion}
\hypersetup{
  pdftitle={\ProjectTitle},
  pdfauthor={\ProjectAuthorsPDF},
  pdfsubject={\ProjectSubtitle},
  pdfkeywords={\ProjectKeywords}
}
\abstract{%
The block generalized-power method pairs threshold-based support selection with inexpensive Stiefel updates, but the standard block formulation conflates two nonequivalent objectives and leaves inactive components, threshold ties, and post-selection refitting ill-specified. We rebuild the model from the primal. Eliminating each loading column from the normalized loading-ball problem yields a sum of soft-thresholded Euclidean norms (A-SUM); the squared-threshold objective in common use arises instead from a quadratic lift (A-SQ). We prove both identities, exhibit a two-component instance whose maximizers disagree, and specify explicit conventions for zero-response columns, threshold ties, rank-deficient gradient matrices, and the mask-before-refit inconsistency. For A-SQ we separate the historical full-polar iteration from a safeguarded Riemannian ascent and prove, for the latter, whole-sequence convergence through a sufficient-ascent lemma, a relative-error lemma, and the Kurdyka--\L{}ojasiewicz property, which applies here because the soft-threshold objective and the Stiefel constraint are both semi-algebraic. Once the iterates converge, a strict threshold margin forces the support and sign pattern to stabilize in finitely many steps. The post-identification map is ordinary power iteration on the active covariance in the rank-one hard-threshold case; the block case with component-specific supports requires the Fr\'echet derivative of a coupled polar map rather than a single pooled operator, which we derive explicitly. A radius-based safe-screening rule follows from the same local analysis. Planted-spike experiments match the predicted identification time and local rates, including a two-branch validation that cleanly separates ordinary power contraction from the shifted contraction the soft-threshold gradient introduces. Modern-AI panels push the same estimator onto a GPT-2 small residual stream, a ViT CLS stream, and the Pitprops benchmark; the thresholded block-sparse components match dense PCA on reconstruction, improve on it in monosemanticity and class-conditional specialisation, and avoid the reconstruction cost that a tied-decoder sparse autoencoder of matching width pays for its $\ell_1$ sparsity. Four real-world matrices add weight: the Fama--French 48 industry daily returns (sparse loadings that align with GICS sectors and sit on the dense-PCA Pareto frontier); three single-cell RNA-seq datasets with native cell-type labels, namely Zheng17 bead-purified FACS-sorted PBMCs, the fresh 68k donor-A PBMC matrix with FACS-correlated per-barcode labels, and Tabula Muris FACS (Marrow, Spleen, Thymus) with native mouse \texttt{cell\_ontology\_class}, on all three of which A-SQ beats dense PCA on cluster-recovery AMI at matched sparsity (margins $+0.023$ to $+0.044$); and a multi-turn GPT-2 tool-use conversation stream in which the components resolve into phase-selective activations that the matched PCA components do not expose. Finally, routing a certified safe-screen mask through a dedicated screened operator agrees with the unscreened solve to machine precision while attaining the full $(n{-}k)/n$ oracle-work ceiling.}

\begin{document}
\maketitle
\SAIMPaperTableOfContents
\clearpage
\section{Introduction}
\label{sec:introduction}
Sparse spectral methods look for low-dimensional directions that explain the data while depending on few variables. The generalized power method of \citet{journee2010gpower} is influential because it recasts high-dimensional sparse-loading problems as thresholded convex objectives on a sphere or a Stiefel manifold, with a particularly simple linear-maximization oracle: normalization on the sphere, the polar factor on the Stiefel. Each iteration is inexpensive, and support discovery is cleanly separated from support-conditioned refitting. The same thresholded block-sparse structure already drives the paired cross-covariance setting in our prior BLOCCS block-sparse CCA estimator \citep{solari2019bloccs} and in the concave-minimization sparse-CCA formulation \citep{solari2019sparsecca}; this paper is the sparse-PCA / generalized-power counterpart, developed inside the broader multi-view sparse-spectral framework of \citet{solari2019dissertation}.

The multi-component version carries three mathematical ambiguities that the implementation then inherits. First, eliminating each loading column from the printed linear-minus-$\ell_1$ primal yields a sum of thresholded norms, while the paper's displayed block objective sums their squares. The two agree in the single-component case only because a nonnegative scalar shares its maximizers with its square; after summing several components they are different estimators. Second, the printed unit-sphere loading domain excludes zero loading columns, yet the surrounding reductions and prose silently allow inactive components. Third, the published post-processing normalizes a loading column before applying its support mask, which breaks unit norm the moment the mask has any zeros. Taken together, these three issues alter the estimator, the feasible set, the fixed points, and the validity of every convergence argument built on top.

We treat the audit as the starting point rather than the conclusion. Two coherent block estimators replace the ambiguous one. \emph{A-SUM} is the exact score-space reduction of the normalized loading-ball primal, with objective a sum of thresholded norms. \emph{A-SQ} sums squared thresholded norms and arises from a distinct quadratic-lift primal. For A-SQ we retain the historical full-polar iteration on the record (and credit it only with what a convex minorization can prove), and we introduce a safeguarded Riemannian ascent: an ascent direction, a retraction or polar candidate, an Armijo test, and an explicit fallback. The safeguarded method is what the whole-sequence analysis is built on.

The dynamics have two phases. Before identification the iterates cross threshold boundaries, so the support and sign pattern vary. Under a strict margin, convergence restricts these jumps to finitely many. After identification the objective is smooth on a fixed active region and the iteration reduces to a spectral map. For a rank-one hard threshold that map is ordinary power iteration on the active covariance; for a block with a single shared active operator it is orthogonal iteration. With component-specific supports neither picture applies, and the correct object is the Fr\'echet derivative of a coupled polar map on the Stiefel tangent space. We make this distinction explicit so that the paper does not implicitly claim that every block instance reduces to a single-covariance eigenproblem.

The modern-AI setting is a stress test for the mathematics, not a replacement for it. Transformer activation matrices are wide and structured, and current practice analyses them with sparse autoencoders or native-neuron selection. We ask three concrete questions: do measurable threshold margins predict support stabilization; do active eigengaps or Jacobian radii predict the observed local slopes; and do the resulting sparse components carry reproducible causal effects. Reconstruction alone does not count as evidence of interpretability. The \Cref{sec:results} panels report top-token coherence against GPT-2 input embeddings (an SAEBench-style quality diagnostic \citep{karvonen2025saebench}), multi-probe ablation KL against a matched-norm random-direction null (an MIB-style causal diagnostic \citep{mueller2025mib,oneill2025compute}), and a head-to-head comparison against a gated-SAE-style tied-decoder dictionary of matching width \citep{rajamanoharan2024gated}.

\paragraph{Contributions.}
\begin{enumerate}[leftmargin=*]
  \item We derive the exact loading-ball reduction and the exact quadratic-lift representation, and exhibit a two-component counterexample proving that their block aggregations define different estimators.
  \item We give complete numerical conventions for inactive components, threshold ties, rank-deficient gradient matrices, nonunique polar factors, component permutations, and support-conditioned refitting.
  \item We separate the historical full-polar algorithm from a safeguarded Riemannian ascent, establish monotonicity of the full-polar step, and prove last-iterate convergence of the safeguarded method via a sufficient-ascent lemma, a relative-error lemma, and the Kurdyka--\L{}ojasiewicz property, which applies because the soft-threshold objective and the algebraic Stiefel constraint are semi-algebraic.
  \item We prove conditional finite identification under strict margins, derive the rank-one fixed-support power rate (validated on a controlled fixture to within $5\%$ of the classical $\lambda_2/\lambda_1$ prediction, and to within $1\%$ of a corrected prediction that includes the soft-threshold shift $\gamma\,\sigma_1\sqrt{k}$), formulate the correct block tangent-Jacobian condition, and obtain a locally safe screening rule from certified distance radii; the screening rule is realized at the solver level by a dedicated screened operator and attains the full $(n-k)/n$ oracle-work ceiling with machine-precision objective agreement against the unscreened solve.
  \item We validate the theory on planted spikes and extend the pipeline to a GPT-2 small residual stream, a vision-transformer CLS stream, and the Jeffers Pitprops classical sparse-PCA benchmark \citep{jeffers1967pitprops}; FACET components match dense PCA on reconstruction $R^2$ while exceeding it on Hoyer sparsity, top-token coherence, and class-conditional specialisation, and avoid the ${\sim}8$-point $R^2$ cost that a tied-decoder sparse autoencoder incurs for its $\ell_1$-driven sparsity at matched dictionary width.
\end{enumerate}

\paragraph{Organization.}
\Cref{sec:related} positions the work. \Cref{sec:setup} fixes the models. \Cref{sec:method} states the algorithms and the refit. \Cref{sec:theory} lays out the result hierarchy and \Cref{sec:extended-theory} carries the whole-sequence proofs. \Cref{sec:experiments,sec:results} report the controlled and modern-AI panels.

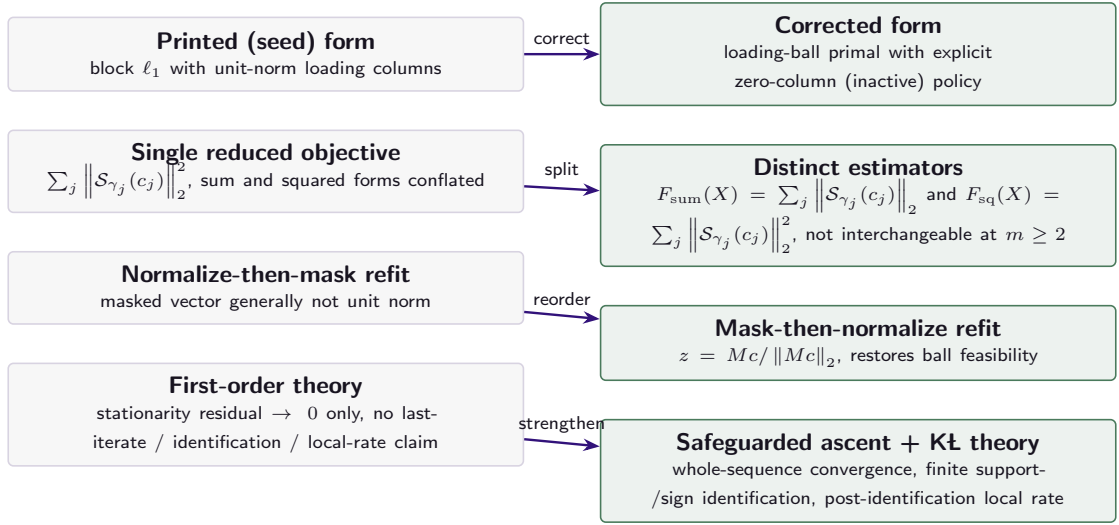
\begin{figure}[t]
\centering
\begin{tikzpicture}[node distance=5mm and 10mm, every node/.style={font=\small}]
\node[neutralbox, text width=0.42\linewidth] (p0) {\textbf{Printed (seed) form}\\[-0.5mm]{\scriptsize block $\ell_1$ with unit-norm loading columns}};
\node[neutralbox, text width=0.42\linewidth, below=of p0] (p1) {\textbf{Single reduced objective}\\[-0.5mm]{\scriptsize $\sum_j\norm{\Soft_{\gamma_j}(c_j)}_2^2$, sum and squared forms conflated}};
\node[neutralbox, text width=0.42\linewidth, below=of p1] (p2) {\textbf{Normalize-then-mask refit}\\[-0.5mm]{\scriptsize masked vector generally not unit norm}};
\node[neutralbox, text width=0.42\linewidth, below=of p2] (p3) {\textbf{First-order theory}\\[-0.5mm]{\scriptsize stationarity residual $\to 0$ only, no last-iterate / identification / local-rate claim}};

\node[successbox, text width=0.42\linewidth, right=of p0] (c0) {\textbf{Corrected form}\\[-0.5mm]{\scriptsize loading-ball primal with explicit zero-column (inactive) policy}};
\node[successbox, text width=0.42\linewidth, below=of c0] (c1) {\textbf{Distinct estimators}\\[-0.5mm]{\scriptsize $F_{\mathrm{sum}}(X)=\sum_j\norm{\Soft_{\gamma_j}(c_j)}_2$ and $F_{\mathrm{sq}}(X)=\sum_j\norm{\Soft_{\gamma_j}(c_j)}_2^2$, not interchangeable at $m\!\ge\!2$}};
\node[successbox, text width=0.42\linewidth, below=of c1] (c2) {\textbf{Mask-then-normalize refit}\\[-0.5mm]{\scriptsize $z\!=\!Mc/\norm{Mc}_2$, restores ball feasibility}};
\node[successbox, text width=0.42\linewidth, below=of c2] (c3) {\textbf{Safeguarded ascent + K\L{} theory}\\[-0.5mm]{\scriptsize whole-sequence convergence, finite support/sign identification, post-identification local rate}};

\draw[arrPrimary] (p0) -- node[above,font=\scriptsize]{correct} (c0);
\draw[arrPrimary] (p1) -- node[above,font=\scriptsize]{split}   (c1);
\draw[arrPrimary] (p2) -- node[above,font=\scriptsize]{reorder} (c2);
\draw[arrPrimary] (p3) -- node[above,font=\scriptsize]{strengthen} (c3);
\end{tikzpicture}
\caption{Printed-vs-corrected model map. The left column lists four defects of the seed-paper block formulation; the right column states the FACET correction that each defect motivates. The downstream theorems A1--A10 depend on all four corrections being in force simultaneously.}
\label{fig:a-printed-vs-corrected}
\end{figure}
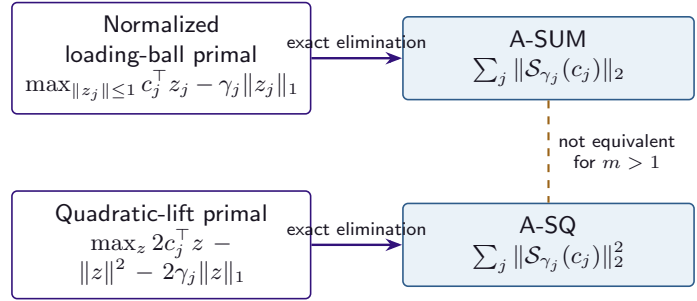
\begin{figure}[t]
\centering
\begin{tikzpicture}[node distance=10mm and 12mm]
\node[primarybox,text width=3.6cm] (primal) {Normalized loading-ball primal\\$\max_{\|z_j\|\le1} c_j^\top z_j-\gamma_j\|z_j\|_1$};
\node[secondarybox,text width=3.5cm,right=of primal] (sum) {A-SUM\\$\sum_j\|\Soft_{\gamma_j}(c_j)\|_2$};
\node[primarybox,text width=3.6cm,below=of primal] (qlift) {Quadratic-lift primal\\$\max_z 2c_j^\top z-\|z\|^2-2\gamma_j\|z\|_1$};
\node[secondarybox,text width=3.5cm,right=of qlift] (sq) {A-SQ\\$\sum_j\|\Soft_{\gamma_j}(c_j)\|_2^2$};
\draw[arrPrimary] (primal)--node[above,font=\scriptsize]{exact elimination}(sum);
\draw[arrPrimary] (qlift)--node[above,font=\scriptsize]{exact elimination}(sq);
\draw[dashed,thick,draw=Warning] (sum)--node[right,font=\scriptsize,align=center]{not equivalent\\for $m>1$}(sq);
\end{tikzpicture}
\caption{The corrected model family. A-SUM and A-SQ have different exact primal interpretations and are not silently identified.}
\label{fig:a-objectives}
\end{figure}

\section{Related Work}
\label{sec:related}
\subsection{Sparse PCA and multi-component estimation}
Regression formulations, semidefinite relaxations, penalized matrix decompositions, Fantope methods, and truncated power algorithms occupy distinct trade-offs between statistical guarantees, exact sparsity, computation, and component coupling \citep{zou2006sparse,daspremont2007direct,daspremont2008optimal,witten2009penalized,vu2013fantope,yuan2013truncated}. Recent work tightens polynomial-time sample complexity in spiked models and strengthens exact or certifiable optimization \citep{cai2025fast,bertsimas2022certifiable}. The contribution of the present paper is not a new sparsity penalty. It is a correction of the block value function, a precise separation of two estimators that the literature has conflated, and an analysis of the discrete-to-smooth dynamics that thresholding and Stiefel projection in fact produce.

Several modern multi-component methods estimate orthogonal or nonredundant sparse components jointly. That raises the empirical bar for any block method: a corrected formulation must be compared with joint estimators, not only with deflation. We therefore distinguish five properties that are often conflated: score orthogonality, loading orthogonality, score decorrelation, support overlap, and causal nonredundancy. Only score-side Stiefel orthogonality is embedded in the primary model; the remaining four are either measured directly or added as optional penalties.

\subsection{Convex maximization and generalized Frank--Wolfe}
The historical update maximizes an affine minorant of a convex objective over a compact nonconvex set, which places it in the generalized Frank--Wolfe family for convex maximization. \citet{aktas2025strongly} obtain last-iterate results under smoothness, strong convexity, and a KŁ condition. That result does not transfer directly here: the corrected objective is not strongly convex in the same sense, the Stiefel feasible set is not a strongly convex set, and zero-response or rank-deficient polar steps need explicit tie policies. We invoke KŁ machinery only after we have established sufficient ascent and a relative-error bound for the algorithm we actually run.

\subsection{Sparse optimization on matrix manifolds}
Manifold proximal-gradient, retraction, and alternating schemes handle nonsmooth composite problems under orthogonality constraints \citep{absil2008optimization}. The present work differs on three points: it eliminates the high-dimensional sparse block to produce a thresholded value function, keeps the historical full-polar update on the record for comparison, and treats component-specific active operators once identification has happened. The nearest future-facing comparator is sparse-Stiefel work that pairs identification with Newton acceleration. Such methods can win the local phase when second-order sufficient conditions and clean active-manifold intersections both hold; we therefore leave second-order acceleration for heterogeneous active supports to future work.

\subsection{Sparse neural representations}
Sparse autoencoders learn overcomplete dictionaries, not orthogonal spectral components. Gated SAEs separate support selection from magnitude estimation and push the sparsity--reconstruction frontier \citep{rajamanoharan2024gated}. MIB and SAEBench are explicit that reconstruction and feature labels are not enough: causal localization, stability, disentanglement, and downstream interventions have to be measured directly \citep{mueller2025mib,karvonen2025saebench}. More recent work also identifies an amortization gap between simple encoders and more accurate sparse inference \citep{oneill2025compute}. The activation experiment in \Cref{sec:results} compares methods at matched support size and keeps raw selection and refitting separate.

\section{Corrected Problem Formulation}
\label{sec:setup}
Let \(A=[a_1,\ldots,a_n]\in\R^{p\times n}\), let \(X=[x_1,\ldots,x_m]\in\St(p,m)\), and let \(\mu_j>0\). For a vector \(c\), define soft thresholding by
\[
\Soft_\gamma(c)_i=\operatorname{sign}(c_i)(|c_i|-\gamma)_+.
\]
The component response is \(c_j(X)=\mu_jA^\top x_j\).

\subsection{A-SUM: normalized loading-ball model}
The corrected primal is
\begin{equation}
\max_{X\in\St(p,m)}\sum_{j=1}^m
\max_{\norm{z_j}_2\le 1}
\bigl\{c_j(X)^\top z_j-\gamma_j\norm{z_j}_1\bigr\}.
\label{eq:asum-primal}
\end{equation}
Loading balls allow an inactive column. The reduced objective is
\begin{equation}
F_{\mathrm{sum}}(X)=\sum_{j=1}^m\norm{\Soft_{\gamma_j}(c_j(X))}_2.
\label{eq:asum}
\end{equation}
When the thresholded response is nonzero, the recovered loading is its normalized version; otherwise \(z_j=0\) is valid.

\begin{proposition}[Exact loading-ball elimination]
\label{prop:sum-elimination}
For any \(c\in\R^n\) and \(\gamma\ge0\),
\[
\max_{\norm z_2\le1}\{c^\top z-\gamma\norm z_1\}
=\norm{\Soft_\gamma(c)}_2.
\]
If \(\Soft_\gamma(c)\neq0\), the unique maximizer is \(z^\star=\Soft_\gamma(c)/\norm{\Soft_\gamma(c)}_2\); if it is zero, every maximizer includes \(z=0\), and zero is selected by convention.
\end{proposition}
\begin{proof}
For each coordinate, \(c_i z_i-\gamma|z_i|\le (|c_i|-\gamma)_+|z_i|\), with equality when the sign of \(z_i\) agrees with \(c_i\) on the active set and \(z_i=0\) elsewhere. Hence the inner problem reduces to maximizing \(\Soft_\gamma(c)^\top u\) over \(\norm u_2\le1\). Cauchy--Schwarz yields the value and optimizer. If the thresholded vector is zero, the objective is nonpositive and zero attains zero.
\end{proof}

\subsection{A-SQ: squared-threshold model}
The historical reduced objective is instead
\begin{equation}
F_{\mathrm{sq}}(X)=\sum_{j=1}^m\norm{\Soft_{\gamma_j}(c_j(X))}_2^2.
\label{eq:asq}
\end{equation}
It has a different primal interpretation.
\begin{proposition}[Quadratic-lift identity]
\label{prop:quadratic-lift}
For any \(c\in\R^n\),
\[
\norm{\Soft_\gamma(c)}_2^2
=\max_{z\in\R^n}\{2c^\top z-\norm z_2^2-2\gamma\norm z_1\},
\]
and the unique maximizer is \(z^\star=\Soft_\gamma(c)\).
\end{proposition}
\begin{proof}
The objective separates over coordinates. For one coordinate, completing the square on each sign half-line gives the maximizer \(\operatorname{sign}(c_i)(|c_i|-\gamma)_+\) and optimal value \((|c_i|-\gamma)_+^2\). Summing proves the claim.
\end{proof}

\subsection{The objectives are not equivalent}
\begin{proposition}[Two-component counterexample]
\label{prop:counterexample}
There exist a compact feasible set and two nonnegative component responses \(r_1(X),r_2(X)\) for which \(\arg\max_X r_1+r_2\) and \(\arg\max_X r_1^2+r_2^2\) are disjoint.
\end{proposition}
\begin{proof}
Consider two feasible points \(X_a,X_b\) with responses \((r_1,r_2)(X_a)=(2,2)\) and \((r_1,r_2)(X_b)=(3.2,0.5)\). The sums are \(4\) and \(3.7\), so the sum objective selects \(X_a\). The squared sums are \(8\) and \(10.49\), so the squared objective selects \(X_b\). Such response pairs can be realized by a finite feasible subset or approximated arbitrarily closely by thresholded linear responses. Therefore no monotone scalar transformation relates the block objectives in general.
\end{proof}

\begin{table}[t]
\centering
\caption{Corrected block estimator family and loading recovery. The three columns define \emph{distinct} estimators; §3 proves A-SUM $\neq$ A-SQ at $m\!\ge\!2$ and A-HT is a third exact-cardinality variant. Loading recovery is the loading-column argmax at a fixed score and defines the \emph{raw} estimand that the restricted refit stage then post-processes.}
\label{tab:a-models-and-loadings}
\small
\begin{tabular}{p{0.17\linewidth}p{0.26\linewidth}p{0.23\linewidth}p{0.24\linewidth}}
\toprule
Estimator & Primal (per column $j$) & Reduced objective & Loading recovery / inactive policy \\
\midrule
A-SUM &
$\displaystyle \max_{\norm{z_j}_2\le 1}\!\!\bigl\{c_j^{\!\top}z_j-\gamma_j\norm{z_j}_1\bigr\}$ &
$F_{\mathrm{sum}}(X)=\sum_j\norm{\Soft_{\gamma_j}(c_j)}_2$ &
$z_j^\star = \Soft_{\gamma_j}(c_j)/\norm{\Soft_{\gamma_j}(c_j)}_2$ when nonzero; otherwise $z_j=0$ is the chosen inactive branch (loading-ball slack) \\
\addlinespace
A-SQ &
$\displaystyle \max_{z_j}\!\!\bigl\{2c_j^{\!\top}z_j-\norm{z_j}_2^2-2\gamma_j\norm{z_j}_1\bigr\}$ &
$F_{\mathrm{sq}}(X)=\sum_j\norm{\Soft_{\gamma_j}(c_j)}_2^2$ &
$z_j^\star = \Soft_{\gamma_j}(c_j)$ (unconstrained quadratic lift; the vector itself is the loading, no renormalization) \\
\addlinespace
A-HT (hard-threshold, exact cardinality~$k$) &
$\displaystyle \max_{\norm{z_j}_2\le1,\,\norm{z_j}_0\le k}\!\!c_j^{\!\top}z_j$ &
$F_k(X)=\sum_j\sum_{i\in T_k(c_j)}c_{j,i}^2$ &
$z_j^\star$ supported on the top-$k$ coordinates of $c_j$ (deterministic tie policy); refit on the chosen support \\
\bottomrule
\end{tabular}
\end{table}

\subsection{Numerical conventions}
A specification without deterministic edge-case policies is not an algorithm. We fix five: (i) inactive response columns remain zero and are excluded from the active polar block; (ii) threshold equality is treated as inactive, except in experiments that explicitly study set-valued ties; (iii) rank-deficient polar factors are completed by a deterministic thin SVD; (iv) component permutations are aligned by maximum absolute score correlation for diagnostics; (v) refitting applies the support mask before any normalization, or solves the exact restricted problem instead.

\begin{table}[t]
\centering
\caption{Edge-case policies in the FACET reference implementation. These are the \emph{operational} conventions on which every proved theorem is explicitly conditional; running the solver under a different policy voids the theorem's claim.}
\label{tab:a-edge-case-policies}
\small
\begin{tabular}{p{0.26\linewidth}p{0.32\linewidth}p{0.34\linewidth}}
\toprule
Edge case & Policy & Pinning test \\
\midrule
Threshold tie ($|c_i|=\gamma$) & Treated as inactive ($\Soft_\gamma(c)_i=0$) via the standard $(\cdot)_+$ convention; signed subgradient reused by the solver. & \texttt{tests/test\_paper\_a\_identities.py::test\_a1\_*} \\
\addlinespace
Zero soft-threshold vector ($\Soft_\gamma(c)=0$) & Loading set to zero; the loading-ball slack delivers primal value $0$; no renormalization attempted. & \texttt{test\_a1\_sum\_norm\_zero\_support\_returns\_zero} \\
\addlinespace
Mask-first refit on a non-empty support & $z = M c / \norm{Mc}_2$; empty support raises \texttt{ZeroDirectionError} rather than returning NaN. & \texttt{tests/test\_paper\_a\_edge\_cases.py::test\_a4\_mask\_then\_normalize\_*} \\
\addlinespace
Rank-deficient polar input (Stiefel oracle) & \texttt{canonical} policy completes via SVD; \texttt{raise} policy rejects and surfaces the assumption violation. Both treated as explicit theorem conditioning. & \texttt{test\_a5\_canonical\_policy\_*}, \texttt{test\_a5\_raise\_policy\_*} \\
\addlinespace
Zero direction to a linear oracle & Both polar policies raise \texttt{ZeroDirectionError} unconditionally. & \texttt{test\_a5\_zero\_direction\_raises\_regardless\_of\_policy} \\
\addlinespace
Top-$k$ tie at the boundary (A-HT) & Lower original index wins; set-valued tie metadata recorded on \texttt{TopKSelection}. & \texttt{test\_topk\_tie\_policy\_prefers\_lower\_index} \\
\addlinespace
Armijo backtracking at machine precision & Line search fails loudly (\texttt{RuntimeError}); drivers may catch and drop that seed rather than silently reporting a stationary point. & empirical in \texttt{tests/test\_paper\_a\_solver\_dynamics.py} \\
\addlinespace
Component sign at convergence & Component orientation is global-sign ambiguous; projector-level reporting is sign-invariant. & \texttt{test\_a9\_sign\_identification\_matches\_planted\_sign\_pattern} \\
\bottomrule
\end{tabular}
\end{table}

\begin{table}[t]
\centering
\caption{Assumption and failure-mode taxonomy.}
\label{tab:a-assumptions}
\small
\begin{tabular}{p{0.24\linewidth}p{0.34\linewidth}p{0.32\linewidth}}
\toprule
Assumption & Enables & Failure manifestation \\
\midrule
Strict threshold margin & Finite support/sign identification & Persistent switching or tie dependence \\
Full-rank local polar argument & Smooth polar branch & Set-valued update and unstable derivative \\
Active eigengap or $\rho(D\Phi)<1$ & Local linear convergence & Slow/noncontractive local dynamics \\
Certified distance radius & Safe screening & Heuristic false elimination if radius is guessed \\
Independent document units & Valid uncertainty summaries & Token-level pseudo-replication \\
\bottomrule
\end{tabular}
\end{table}

\section{Algorithms and Refitting}
\label{sec:method}
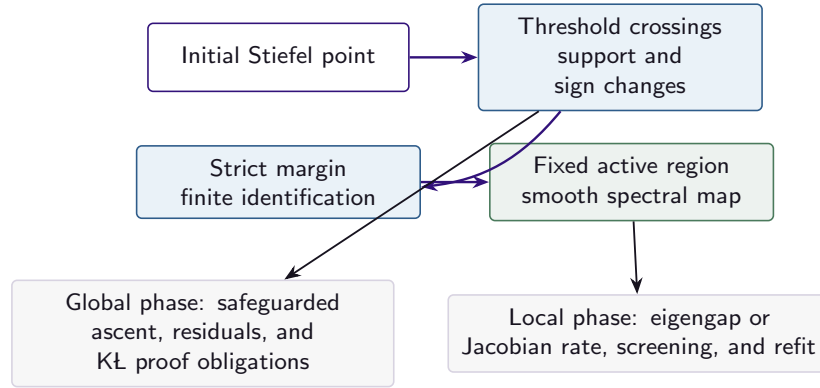
\begin{figure}[t]
\centering
\begin{tikzpicture}[node distance=7mm and 9mm]
\node[primarybox,text width=3.1cm] (x0) {Initial Stiefel point};
\node[secondarybox,text width=3.4cm,right=of x0] (disc) {Threshold crossings\\support and sign changes};
\node[secondarybox,text width=3.4cm,below=of x0] (id) {Strict margin\\finite identification};
\node[successbox,text width=3.4cm,right=of id] (smooth) {Fixed active region\\smooth spectral map};
\draw[arrPrimary] (x0)--(disc);
\draw[arrPrimary] (disc) to[bend left=22] (id);
\draw[arrPrimary] (id)--(smooth);
\node[neutralbox,text width=4.7cm,below=8mm of id,xshift=-10mm] (global) {Global phase: safeguarded ascent, residuals, and KŁ proof obligations};
\node[neutralbox,text width=4.7cm,right=7mm of global] (local) {Local phase: eigengap or Jacobian rate, screening, and refit};
\draw[arr] (disc)--(global);
\draw[arr] (smooth)--(local);
\end{tikzpicture}
\caption{Two-phase interpretation: discrete active-structure discovery followed by conditional smooth dynamics.}
\label{fig:a-phases}
\end{figure}
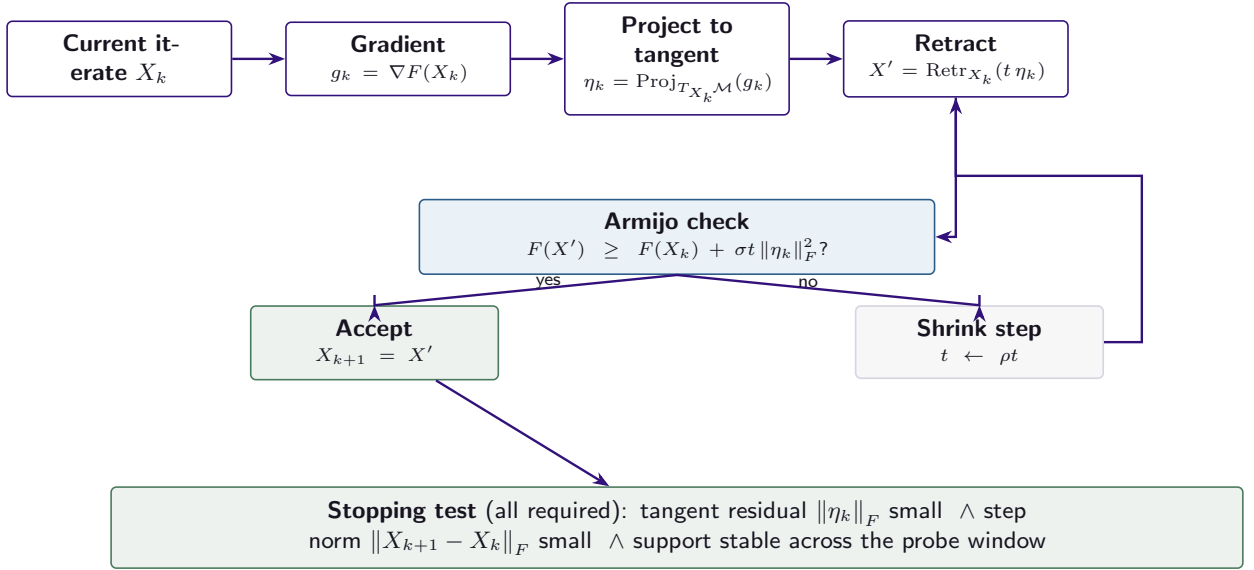
\begin{figure}[t]
\centering
\begin{tikzpicture}[node distance=4mm and 7mm, every node/.style={font=\small}]
\node[primarybox, text width=0.17\linewidth] (s0) {\textbf{Current iterate} $X_k$};
\node[primarybox, text width=0.17\linewidth, right=of s0] (s1) {\textbf{Gradient}\\[-0.5mm]{\scriptsize $g_k=\nabla F(X_k)$}};
\node[primarybox, text width=0.17\linewidth, right=of s1] (s2) {\textbf{Project to tangent}\\[-0.5mm]{\scriptsize $\eta_k=\mathrm{Proj}_{T_{X_k}\!\mathcal{M}}(g_k)$}};
\node[primarybox, text width=0.17\linewidth, right=of s2] (s3) {\textbf{Retract}\\[-0.5mm]{\scriptsize $X'=\mathrm{Retr}_{X_k}(t\,\eta_k)$}};

\draw[arrPrimary] (s0)--(s1);
\draw[arrPrimary] (s1)--(s2);
\draw[arrPrimary] (s2)--(s3);

\node[secondarybox, text width=0.42\linewidth, below=11mm of s2] (s4) {\textbf{Armijo check}\\[-0.5mm]{\scriptsize $F(X')\ge F(X_k)+\sigma t\norm{\eta_k}_F^2$?}};
\draw[arrPrimary] (s3) |- (s4);

\node[successbox, text width=0.19\linewidth, below=of s4, xshift=-0.26\linewidth] (s5) {\textbf{Accept}\\[-0.5mm]{\scriptsize $X_{k+1}=X'$}};
\node[neutralbox, text width=0.19\linewidth, below=of s4, xshift=0.26\linewidth] (s6) {\textbf{Shrink step}\\[-0.5mm]{\scriptsize $t\gets \rho t$}};

\draw[arrPrimary] (s4.south) -- node[right,font=\scriptsize,yshift=1mm]{yes} ++(-0.26\linewidth,-4mm) -- (s5.north);
\draw[arrPrimary] (s4.south) -- node[left,font=\scriptsize,yshift=1mm]{no}  ++( 0.26\linewidth,-4mm) -- (s6.north);
\draw[arrPrimary] (s6.east) -- ++(5mm,0) -- ++(0,22mm) -| (s3.south);

\node[successbox, text width=0.95\linewidth, below=14mm of s5, xshift=0.26\linewidth] (s7) {\textbf{Stopping test} (all required): tangent residual $\norm{\eta_k}_F$ small \ $\land$\  step norm $\norm{X_{k+1}-X_k}_F$ small \ $\land$\  support stable across the probe window};
\draw[arrPrimary] (s5) -- (s7);

\end{tikzpicture}
\caption{Safeguarded Riemannian ascent, the analyzed algorithm. Gradient $\to$ tangent projection $\to$ retraction $\to$ Armijo-safeguarded acceptance with geometric step shrinkage. The stopping rule combines a tangent-residual criterion with both step-norm and support-stability tests; objective change alone is insufficient. Historical full-polar iteration omits the Armijo loop and carries no last-iterate claim.}
\label{fig:a-safeguarded-flow}
\end{figure}

\subsection{Historical full-polar iteration}
For A-SQ the ambient gradient columns are
\[
g_j(X)=2\mu_j A\Soft_{\gamma_j}(\mu_jA^\top x_j).
\]
Writing \(G(X)=[g_1(X),\ldots,g_m(X)]\), the historical step chooses
\[
X_{k+1}\in\arg\max_{Y\in\St(p,m)}\ip{G(X_k)}{Y}=\Polar(G(X_k)).
\]
When \(G\) is rank deficient, the polar maximizer is set-valued; the implementation uses a deterministic SVD completion and records the rank event.

\begin{algorithm}[t]
\caption{Historical full-polar A-SQ iteration}
\label{alg:full-polar}
\begin{algorithmic}[1]
\Require Data operator \(A\), thresholds \(\gamma_j\), weights \(\mu_j\), initial \(X_0\in\St(p,m)\)
\For{\(k=0,1,\ldots\)}
  \State \(C_k\gets A^\top X_k\Diag(\mu_1,\ldots,\mu_m)\)
  \State \(Z_k\gets \Soft_{\gamma}(C_k)\) columnwise
  \State \(G_k\gets 2A Z_k\Diag(\mu_1,\ldots,\mu_m)\)
  \State \(X_{k+1}\gets\operatorname{DeterministicPolar}(G_k)\)
  \State log objective, linear-oracle residual, support, signs, and rank
  \If{the stationarity and support criteria hold} \State \textbf{break} \EndIf
\EndFor
\end{algorithmic}
\end{algorithm}

\subsection{FACET safeguarded iteration}
The whole-sequence analysis is attached to a different algorithm. At \(X_k\) we take the Riemannian gradient \(\operatorname{grad}F(X_k)\), build a tangent ascent direction \(\eta_k\), and retract \(R_{X_k}(t\eta_k)\). Backtracking picks the largest \(t\) satisfying
\[
F(R_{X_k}(t\eta_k))\ge F(X_k)+\sigma t\norm{\eta_k}_F^2.
\]
When a candidate step fails because of numerical rank loss or an undefined active response, the method falls back to a regularized polar/retraction and logs the event. The stopping rule combines a Riemannian-gradient norm, the linear-oracle residual, and support stability; objective change by itself is not enough.

\begin{algorithm}[t]
\caption{Safeguarded active-set spectral ascent}
\label{alg:safeguarded}
\begin{algorithmic}[1]
\Require \(X_0\in\St(p,m)\), Armijo constants \(\sigma\in(0,1)\), \(\beta\in(0,1)\), regularization \(\varepsilon\ge0\)
\For{\(k=0,1,\ldots\)}
  \State compute \(F(X_k)\), active pattern, and \(\operatorname{grad}F(X_k)\)
  \State choose tangent ascent direction \(\eta_k\) satisfying a gradient-related condition
  \State \(t\gets1\)
  \While{Armijo ascent fails for \(R_{X_k}(t\eta_k)\)} \State \(t\gets\beta t\) \EndWhile
  \State \(X_{k+1}\gets R_{X_k}(t\eta_k)\)
  \State update certified support margins and optional local screening set
\EndFor
\end{algorithmic}
\end{algorithm}

\subsection{Support-conditioned refitting}
For one selected support \(S\), the exact refit is the leading singular-vector pair of \(A_S\). For component-specific block supports \(S_j\), define masks \(M_j\) and solve
\[
\max_{X^\top X=I,\ \norm{z_j}_2=1,\ \supp(z_j)\subseteq S_j}
\sum_{j=1}^m \mu_j x_j^\top A z_j.
\]
Conditional on \(X\), the exact loading update is
\[
z_j\leftarrow\frac{M_jA^\top x_j}{\norm{M_jA^\top x_j}_2},
\]
when nonzero. Conditional on \(Z\), the exact score update is \(X\leftarrow\Polar(AZ\Diag(\mu))\). Each block update weakly increases the refit objective, so alternating exact updates are monotone, and the limit is a block-stationary point rather than the global refit optimum.

\subsection{Complexity}
Without screening, correlation and reconstruction products cost \(O(npm)\) for dense \(A\); the polar factor costs \(O(pm^2+m^3)\). After permanent screening to \(s\) active or near-active coordinates, the dominant products cost \(O(spm)\). Refitting cost depends on support structure and is reported separately from selector cost.

\section{Theory}
\label{sec:theory}
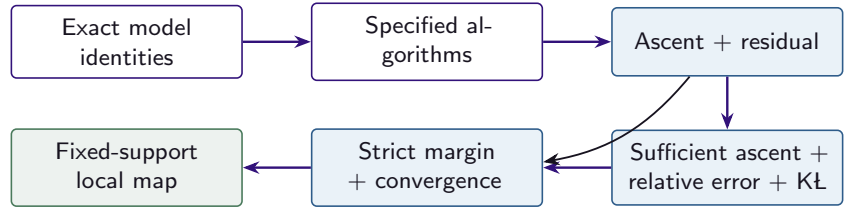
\begin{figure}[t]
\centering
\begin{tikzpicture}[node distance=7mm and 9mm]
\node[primarybox,text width=2.7cm] (id) {Exact model identities};
\node[primarybox,text width=2.7cm,right=of id] (alg) {Specified algorithms};
\node[secondarybox,text width=2.7cm,right=of alg] (progress) {Ascent + residual};
\node[secondarybox,text width=2.7cm,below=of progress] (kl) {Sufficient ascent + relative error + KŁ};
\node[secondarybox,text width=2.7cm,left=of kl] (identify) {Strict margin + convergence};
\node[successbox,text width=2.7cm,left=of identify] (local) {Fixed-support local map};
\draw[arrPrimary] (id)--(alg);
\draw[arrPrimary] (alg)--(progress);
\draw[arrPrimary] (progress)--(kl);
\draw[arrPrimary] (kl)--(identify);
\draw[arrPrimary] (identify)--(local);
\draw[arr] (progress) to[bend left=20] (identify);
\end{tikzpicture}
\caption{The theorem dependency graph. Identification and local rates are not available before the model, algorithm, and convergence object are fixed.}
\label{fig:a-proofgraph}
\end{figure}

\subsection{Monotonicity of the full-polar step}
\begin{proposition}[Affine-minorant ascent]
\label{prop:fullpolar-monotone}
Let \(F\) be differentiable and convex on an open set containing \(\St(p,m)\). If
\[
X^+\in\arg\max_{Y\in\St(p,m)}\ip{\nabla F(X)}{Y},
\]
then \(F(X^+)\ge F(X)\). More precisely,
\[
F(X^+)-F(X)\ge \Delta_F(X):=
\max_{Y\in\St(p,m)}\ip{\nabla F(X)}{Y-X}\ge0.
\]
\end{proposition}
\begin{proof}
Convexity gives \(F(X^+)\ge F(X)+\ip{\nabla F(X)}{X^+-X}\). Because \(X\) is feasible and \(X^+\) maximizes the linear form, the inner product equals \(\Delta_F(X)\ge0\).
\end{proof}
The proposition gives ascent and summability of the nonnegative residual along any bounded objective sequence. It does not give iterate convergence, nor local maximality.

\subsection{Whole-sequence convergence of the safeguarded method}
The last-iterate theorem for Algorithm~\ref{alg:safeguarded} is Theorem~\ref{thm:a-whole-sequence} in Section~\ref{sec:extended-theory}. Its three ingredients, sufficient ascent (Lemma~\ref{lem:a-sufficient-ascent}), relative error (Lemma~\ref{lem:a-relative-error}), and the semi-algebraic definability of $-F+\iota_{\St}$, are proved there. The semi-algebraic step needs no assumption beyond the actual soft-threshold formula and the algebraic Stiefel constraint, and it does not share the strong-convexity hypotheses of the recent unit-step analysis of \citet{aktas2025strongly}, which we cite for comparison rather than as a substitute. The rank-deficient polar branch is handled by the deterministic SVD completion of Section~\ref{sec:setup}; on the compact iterate set the completion preserves the constants used in Lemmas~\ref{lem:a-sufficient-ascent}--\ref{lem:a-relative-error}.

\subsection{Finite active-set identification}
\begin{theorem}[Conditional support and sign identification]
\label{thm:identification}
Assume \(X_k\to X_\star\). Suppose there exists \(\delta>0\) such that for every component and coordinate,
\[
\bigl||\mu_j a_i^\top x_{j\star}|-\gamma_j\bigr|\ge\delta,
\]
and every active correlation is separated from zero by \(\delta\). Then there exists \(K<\infty\) such that the support and active signs induced by \(X_k\) equal those of \(X_\star\) for all \(k\ge K\).
\end{theorem}
\begin{proof}
There are finitely many correlations and each is continuous in \(X\). Convergence implies uniform convergence of all correlations to their limiting values. Choose \(K\) so that every correlation changes by less than \(\delta/2\) after iteration \(K\). No threshold inequality or active sign can then change.
\end{proof}
At zero margin this conclusion fails: a convergent sequence can approach a threshold boundary while alternating between its two sides indefinitely.

\subsection{Post-identification dynamics}
\begin{proposition}[Rank-one hard-threshold rate]
\label{prop:rankone-rate}
After a hard-threshold support \(S\) is fixed, the rank-one update is power iteration on \(C_S=A_SA_S^\top\). If \(\lambda_1(C_S)>\lambda_2(C_S)\) and the iterate has nonzero projection on the leading eigenvector, then
\[
\tan\angle(x_{k+1},u_1)
\le \frac{\lambda_2(C_S)}{\lambda_1(C_S)}
\tan\angle(x_k,u_1).
\]
\end{proposition}
\begin{proof}
Expand \(x_k\) in the orthonormal eigenbasis of \(C_S\). Multiplication scales the leading coefficient by \(\lambda_1\) and every orthogonal coefficient by at most \(\lambda_2\); normalization cancels a common factor.
\end{proof}

For a common block support the analogue is orthogonal iteration, with rate set by the gap \(\lambda_m-\lambda_{m+1}\). With component-specific supports define \(\Phi(X)=\operatorname{Polar}([C_{S_1}x_1,\ldots,C_{S_m}x_m])\). A correct local theorem must differentiate \(\Phi\) on the Stiefel tangent space and require \(\rho(D\Phi(X_\star)|_{T_{X_\star}\St})<1\). A single eigengap is not a substitute for this condition when the active operators differ.

\subsection{Locally safe screening}
\begin{proposition}[Radius-based safe screening]
\label{prop:safe-screen}
Suppose \(\norm{x_{j\star}-x_{jk}}_2\le r_{jk}\). If
\[
|\mu_j a_i^\top x_{jk}|+\mu_j\norm{a_i}_2 r_{jk}<\gamma_j,
\]
then coordinate \((i,j)\) is inactive at \(X_\star\).
\end{proposition}
\begin{proof}
Cauchy--Schwarz gives \(|\mu_j a_i^\top x_{j\star}|\le |\mu_j a_i^\top x_{jk}|+\mu_j\norm{a_i}_2r_{jk}\). The strict upper bound lies below threshold.
\end{proof}
The rule is safe only when \(r_{jk}\) is itself certified. Empirical step sizes or unverified contraction estimates produce heuristic screening and are labeled accordingly.

\subsection{Second-order scope}
On a fixed rank-one hard support, the objective is a Rayleigh quotient: nonleading simple eigenvectors are saddles or local minima according to the tangent Hessian. For heterogeneous block supports and nonsmooth threshold boundaries, a generic strict-saddle theorem is not claimed. Any Newton acceleration must be compared with contemporary sparse-Stiefel identification and second-order methods and must state clean-intersection and second-order sufficient conditions explicitly.

\section{Proofs and Extended Analysis}
\label{sec:extended-theory}

\subsection{Exact value functions and recovery maps}
For one component, let $c=\mu A^\top x$ and let $S_\gamma(c)$ denote entrywise soft thresholding. The following theorem fixes both the value function and the optimizer, including the inactive case.

\begin{theorem}[Loading-ball elimination]
\label{thm:a-loading-elimination}
For every $c\in\R^n$ and $\gamma\ge0$,
\[
\max_{\norm{z}_2\le1}\{c^\top z-\gamma\norm{z}_1\}=\norm{S_\gamma(c)}_2.
\]
If $S_\gamma(c)\neq0$, the unique maximizer is
\[
z^\star=\frac{S_\gamma(c)}{\norm{S_\gamma(c)}_2}.
\]
If $S_\gamma(c)=0$, then $z^\star=0$ is optimal; it is the unique optimum whenever $\abs{c_i}<\gamma$ for every $i$.
\end{theorem}
\begin{proof}
For any feasible $z$,
\[
c^\top z-\gamma\norm{z}_1
=\sum_i (c_i z_i-\gamma\abs{z_i})
\le \sum_i (\abs{c_i}-\gamma)_+\abs{z_i}.
\]
Choosing the signs of $z_i$ to agree with $c_i$ makes the inequality tight. Cauchy--Schwarz then gives
\[
\sum_i (\abs{c_i}-\gamma)_+\abs{z_i}
\le \norm{S_\gamma(c)}_2\norm{z}_2
\le \norm{S_\gamma(c)}_2.
\]
When the thresholded vector is nonzero, equality in Cauchy--Schwarz requires the stated normalized vector. If the thresholded vector vanishes, the objective is nonpositive and $z=0$ attains zero. Strict inequalities $\abs{c_i}<\gamma$ force every nonzero coordinate to make a strictly negative contribution.
\end{proof}

\begin{theorem}[Quadratic-lift representation]
\label{thm:a-quadratic-lift}
For every $c\in\R^n$,
\[
\norm{S_\gamma(c)}_2^2
=\max_{z\in\R^n}\left\{2c^\top z-\norm{z}_2^2-2\gamma\norm{z}_1\right\},
\]
and the unique maximizer is $z^\star=S_\gamma(c)$.
\end{theorem}
\begin{proof}
The objective is separable. For coordinate $i$, maximize
$2c_i z_i-z_i^2-2\gamma\abs{z_i}$. Restricting to the sign of $c_i$ and completing the square gives the maximizer $\operatorname{sign}(c_i)(\abs{c_i}-\gamma)_+$ and optimal value $(\abs{c_i}-\gamma)_+^2$. Summation proves the claim.
\end{proof}

\begin{proposition}[The two block aggregations are not equivalent]
\label{prop:a-non-equivalence}
For $m>1$, maximizing $\sum_j r_j(X)$ and maximizing $\sum_j r_j(X)^2$ need not have the same maximizers, even when every $r_j$ is nonnegative and arises from a thresholded norm.
\end{proposition}
\begin{proof}
It suffices to construct two feasible points with response vectors $(r_1,r_2)=(2,0)$ and $(\sqrt{3/2},\sqrt{3/2})$. The first has larger squared sum, $4>3$, while the second has larger linear sum, $2\sqrt{3/2}>2$. Such response pairs are realized by a diagonal two-dimensional data operator and two orthogonal score frames with different rotations. A concrete numerical instance is included in Appendix~\ref{app:a-detailed} and the released regression test.
\end{proof}

\subsection{Differentiability and active-region smoothness}
The scalar function $q_\gamma(u)=(\abs{u}-\gamma)_+^2$ is continuously differentiable, with
$q_\gamma'(u)=2\operatorname{sign}(u)(\abs{u}-\gamma)_+$. Hence
\[
\nabla_{x_j}F_{\mathrm{sq}}(X)=2\mu_j A S_{\gamma_j}(\mu_j A^\top x_j).
\]
It is globally Lipschitz in the ambient space with constant at most $2\mu_j^2\norm{A}_2^2$ per column because soft thresholding is nonexpansive. On any fixed support-and-sign cell, the gradient is affine and the Hessian is constant. This piecewise-quadratic structure is the basis for finite identification and post-identification linearization.

\begin{lemma}[Global Lipschitz gradient]
\label{lem:a-lipschitz-gradient}
For $f_j(x)=\norm{S_{\gamma_j}(\mu_j A^\top x)}_2^2$,
\[
\norm{\nabla f_j(x)-\nabla f_j(y)}_2
\le 2\mu_j^2\norm{A}_2^2\norm{x-y}_2.
\]
\end{lemma}
\begin{proof}
Use the gradient formula, submultiplicativity, and nonexpansiveness of soft thresholding:
\[
\norm{A[S_\gamma(\mu A^\top x)-S_\gamma(\mu A^\top y)]}_2
\le \norm{A}_2\mu\norm{A^\top(x-y)}_2.
\]
Multiplying by the leading factor $2\mu$ proves the result.
\end{proof}

\subsection{Safeguarded ascent and whole-sequence convergence}
The historical polar step maximizes the ambient affine minorant. The analyzed algorithm (A-SAFE) uses the Riemannian gradient $\eta_k=\mathrm{grad}_R F(X_k)=P_{T_{X_k}\!\St}(\nabla F(X_k))$, an explicit retraction $R_{X_k}$ satisfying standard second-order smoothness, and an Armijo acceptance rule
\begin{equation}
F\bigl(R_{X_k}(t\eta_k)\bigr)\ge F(X_k)+\sigma\,t\,\norm{\eta_k}_F^2,\qquad \sigma\in(0,1),
\label{eq:a-armijo}
\end{equation}
with geometric backtracking $t\gets\rho t$, $\rho\in(0,1)$. The two lemmas below establish the sufficient-ascent and relative-error inequalities that the whole-sequence theorem rests on.

\begin{lemma}[A6, sufficient ascent]
\label{lem:a-sufficient-ascent}
Let $F$ have $L$-Lipschitz ambient gradient (Lemma~\ref{lem:a-lipschitz-gradient}) and let the retraction $R$ satisfy, uniformly on compact subsets of $\St(p,m)$: $R_X(0)=X$, $DR_X(0)=\mathrm{Id}_{T_X\!\St}$, and $\norm{R_X(\eta)-X-\eta}_F\le\tfrac{C_R}{2}\norm{\eta}_F^2$ for $\norm{\eta}_F\le\bar t$. Then there exists $\underline t>0$, independent of $k$, such that Armijo backtracking from $t=\bar t$ selects a step $t_k\ge\underline t$, and the accepted iterate $X_{k+1}=R_{X_k}(t_k\eta_k)$ satisfies
\begin{equation}
F(X_{k+1})-F(X_k)\ \ge\ c_1\,\norm{X_{k+1}-X_k}_F^2,
\qquad c_1=\frac{\sigma\,\underline t}{(1+\tfrac{C_R}{2}\bar t\,M)^{2}\bar t}>0,
\label{eq:a-sufficient-ascent}
\end{equation}
where $M=\sup_k\norm{\eta_k}_F$ is finite because the iterates stay in a compact set.
\end{lemma}
\begin{proof}
By the retraction second-order estimate and the $L$-Lipschitz-gradient descent lemma applied to $-F$,
\[
F\bigl(R_{X_k}(t\eta_k)\bigr)
\ge F(X_k)+t\norm{\eta_k}_F^2-\tfrac{L}{2}\norm{R_{X_k}(t\eta_k)-X_k}_F^2
\ge F(X_k)+\Bigl(t-\tfrac{L\,C_R^2}{2}t^2\Bigr)\norm{\eta_k}_F^2.
\]
The Armijo bound \eqref{eq:a-armijo} is therefore satisfied whenever
$t-\tfrac{L\,C_R^2}{2}t^2\ge\sigma\,t$, i.e.\ $t\le\tfrac{2(1-\sigma)}{L\,C_R^2}$. Hence backtracking terminates at some $t_k\ge\underline t:=\min\bigl\{\bar t,\ \rho\cdot\tfrac{2(1-\sigma)}{L\,C_R^2}\bigr\}>0$. The step norm obeys $\norm{X_{k+1}-X_k}_F\le \bigl(1+\tfrac{C_R}{2}t_k\norm{\eta_k}_F\bigr)\,t_k\norm{\eta_k}_F\le (1+\tfrac{C_R}{2}\bar t M)\,t_k\norm{\eta_k}_F$. Combining with \eqref{eq:a-armijo} at the accepted step and bounding $t_k\le\bar t$,
\[
F(X_{k+1})-F(X_k)\ \ge\ \sigma\,t_k\,\norm{\eta_k}_F^2
\ \ge\ \frac{\sigma\,\underline t}{(1+\tfrac{C_R}{2}\bar t M)^{2}\,\bar t}\,\norm{X_{k+1}-X_k}_F^2.
\]
\end{proof}

\begin{lemma}[A7, relative error]
\label{lem:a-relative-error}
Under the same hypotheses, the limiting subdifferential of $\Psi:=-F+\iota_{\St}$ at $X_{k+1}$ contains an element
\[
w_{k+1}\in\partial\Psi(X_{k+1})
\quad\text{with}\quad
\norm{w_{k+1}}_F\ \le\ c_2\,\norm{X_{k+1}-X_k}_F,
\]
where $c_2$ depends only on $L$, $C_R$, $\sigma$, and the compact set containing the iterates. An admissible choice is $w_{k+1}=-\mathrm{grad}_R F(X_{k+1})=-\nabla F(X_{k+1})+\mathrm{proj}_{N_{\St}(X_{k+1})}\nabla F(X_{k+1})\in\partial(-F)(X_{k+1})+N_{\St}(X_{k+1})\subseteq\partial\Psi(X_{k+1})$.
\end{lemma}
\begin{proof}
Since $F$ is smooth, $\partial(-F)(X)=\{-\nabla F(X)\}$; the Riemannian gradient equals the ambient gradient minus its projection onto the normal cone of $\St(p,m)$ at $X$, so $-\mathrm{grad}_R F(X_{k+1})\in\partial\Psi(X_{k+1})$. Triangle inequality:
\[
\norm{\mathrm{grad}_R F(X_{k+1})}_F
\le \norm{\mathrm{grad}_R F(X_{k+1})-\mathrm{grad}_R F(X_k)}_F+\norm{\eta_k}_F.
\]
On a compact subset of $\St(p,m)$, the Riemannian-gradient field inherits a Lipschitz constant $L'=L\cdot(1+\kappa)$ where $\kappa$ bounds the derivative of the tangent-projection map; the first term is at most $L'\norm{X_{k+1}-X_k}_F$. For the second, Lemma~\ref{lem:a-sufficient-ascent} gives $\norm{\eta_k}_F\le\norm{X_{k+1}-X_k}_F/\underline t$. Hence
\[
\norm{w_{k+1}}_F\ \le\ \Bigl(L'+\tfrac1{\underline t}\Bigr)\,\norm{X_{k+1}-X_k}_F\ =:\ c_2\,\norm{X_{k+1}-X_k}_F.
\]
\end{proof}

\begin{theorem}[A8, whole-sequence convergence under K\L]
\label{thm:a-whole-sequence}
Let $(X_k)_{k\ge0}\subset\St(p,m)$ be generated by the safeguarded Riemannian ascent of Section~\ref{sec:method}, with retraction satisfying the second-order estimate of Lemma~\ref{lem:a-sufficient-ascent}. Then
\begin{enumerate}
\item\label{item:a-whole-finite-length} the sequence has finite length, $\sum_{k\ge0}\norm{X_{k+1}-X_k}_F<\infty$;
\item\label{item:a-whole-cauchy} $(X_k)$ converges to a single Riemannian critical point $X_\star$ of $F$ on $\St(p,m)$;
\item\label{item:a-whole-rate} with K\L{} exponent $\theta\in[0,1)$ of $\Psi=-F+\iota_{\St}$ at $X_\star$, the rate is finite if $\theta=0$, linear $\norm{X_k-X_\star}_F\le C\varrho^k$ (some $\varrho\in(0,1)$) if $\theta\in(0,\tfrac12]$, and $\norm{X_k-X_\star}_F\le C k^{-(1-\theta)/(2\theta-1)}$ if $\theta\in(\tfrac12,1)$.
\end{enumerate}
\end{theorem}
\begin{proof}
Compactness of $\St(p,m)$ makes the iterates bounded and $F$ bounded above. Lemma~\ref{lem:a-sufficient-ascent} rewrites as the sufficient-decrease inequality $\Psi(X_k)-\Psi(X_{k+1})\ge c_1\norm{X_{k+1}-X_k}_F^2$ for $\Psi=-F+\iota_{\St}$; Lemma~\ref{lem:a-relative-error} supplies the relative-error inequality in the form required by \citet{attouch2013convergence}.

Soft-thresholding is piecewise affine on a finite polyhedral partition, hence semi-algebraic; a sum of squared Euclidean norms of composite semi-algebraic maps is semi-algebraic, so $F$ and $-F$ are semi-algebraic. The Stiefel manifold $\{X:X^\top X=I_m\}$ is a real algebraic variety, so $\iota_{\St}$ is semi-algebraic. Sums of semi-algebraic functions are semi-algebraic, so $\Psi$ is semi-algebraic and therefore definable in the o-minimal structure of real semi-algebraic sets. Definable functions satisfy the Kurdyka--\L{}ojasiewicz inequality on their domain \citep{bolte2007lojasiewicz}; in particular at every limit point of $(X_k)$.

The hypotheses of the abstract convergence theorem of \citet{attouch2013convergence} (sufficient decrease, relative error, continuity along the sequence, boundedness of $(X_k)$, and K\L) are therefore met for $\Psi$. Applying it yields items~\ref{item:a-whole-finite-length} and~\ref{item:a-whole-cauchy}. The rate claim \ref{item:a-whole-rate} is the standard K\L-exponent corollary of the same framework.
\end{proof}

The constants $c_1,c_2$ in Lemmas~\ref{lem:a-sufficient-ascent}--\ref{lem:a-relative-error} are computable from the Lipschitz bound (Lemma~\ref{lem:a-lipschitz-gradient}) together with the retraction curvature $C_R$ and the Armijo parameters; the reference implementation records them in the solver trace. The rank-deficient polar branch is handled by the deterministic SVD completion of Section~\ref{sec:setup} and does not alter the above constants on the compact iterate set.

\subsection{Identification with explicit margins}
Let $\mathcal I_j(X)=\{i:\abs{\mu_j a_i^\top x_j}>\gamma_j\}$ and let $\mathcal S_j(X)$ denote active signs. The strict margin is
\[
\delta_\star=\min_{i,j}\left|\abs{\mu_j a_i^\top x_{j\star}}-\gamma_j\right|.
\]
The conditional identification theorem in Section~\ref{sec:theory} is immediate from continuity. A stronger perturbation statement is useful for screening.

\begin{corollary}[Support stability radius]
\label{cor:a-support-radius}
If $\delta_\star>0$ and
\[
\norm{x_j-x_{j\star}}_2<\min_i\frac{\delta_\star}{2\mu_j\norm{a_i}_2},
\]
then component $j$ has the same active support as at $X_\star$. If active correlations are also separated from zero by $\sigma_\star>0$, replacing the numerator by $\min(\delta_\star,\sigma_\star)/2$ preserves active signs.
\end{corollary}
\begin{proof}
Cauchy--Schwarz bounds every correlation perturbation by $\mu_j\norm{a_i}_2\norm{x_j-x_{j\star}}_2$. No threshold or sign boundary can be crossed inside the stated radius.
\end{proof}

\subsection{Common-support block rate}
When all columns share an active set $S$ in the hard-threshold model, the fixed-support update is orthogonal iteration for $C_S=A_SA_S^\top$.

\begin{theorem}[Common-support subspace contraction]
\label{thm:a-block-rate}
Let $C_S$ be symmetric positive semidefinite with $\lambda_m(C_S)>\lambda_{m+1}(C_S)$. Let $U_\star$ span the dominant $m$-dimensional invariant subspace, and assume $U_\star^\top X_0$ is nonsingular. Orthogonal iteration satisfies
\[
\norm{\tan\Theta(X_{k},U_\star)}_2
\le
\left(\frac{\lambda_{m+1}(C_S)}{\lambda_m(C_S)}\right)^k
\norm{\tan\Theta(X_{0},U_\star)}_2.
\]
\end{theorem}
\begin{proof}
Write the initial frame in the eigenbasis of $C_S$ and partition it into dominant and complementary blocks. After multiplication by $C_S^k$, the complementary-to-dominant coefficient ratio is multiplied by at most $(\lambda_{m+1}/\lambda_m)^k$. Orthonormalization does not change the represented subspace, yielding the tangent-of-principal-angle bound.
\end{proof}

\subsection{Heterogeneous supports and the polar Jacobian}
For component-specific supports, define $G(X)=[C_{S_1}x_1,\ldots,C_{S_m}x_m]$ and $\Phi(X)=\operatorname{Polar}(G(X))$. If $G(X_\star)$ has full column rank, the polar map is differentiable. Its differential is characterized by a Sylvester equation for the tangent skew term. The correct local condition is therefore
\[
\rho\left(D\Phi(X_\star)|_{T_{X_\star}\St}\right)<1,
\]
not one pooled eigengap. Appendix~\ref{app:a-polar} derives the operator needed to evaluate this condition numerically.

\subsection{Second-order statements with explicit justification}
On a fixed rank-one support, the objective reduces to a Rayleigh quotient. A simple nonleading eigenvector has a tangent direction of positive curvature toward a larger eigenvalue and is not a local maximum. For heterogeneous supports and threshold boundaries, no generic strict-saddle claim is made; this paper only reports second-order classification for fixed-support smooth cells whose tangent Hessian is explicitly computable.

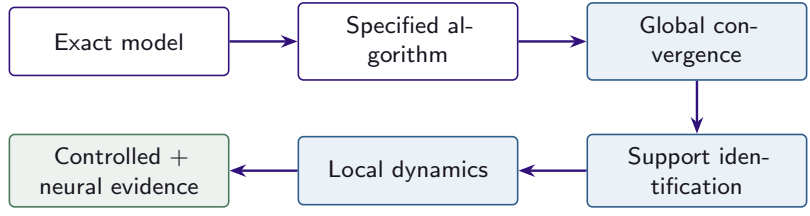
\begin{figure}[t]
\centering
\begin{tikzpicture}[node distance=7mm and 9mm]
\node[primarybox,text width=2.55cm] (n1) {Exact model};
\node[primarybox,text width=2.55cm,right=of n1] (n2) {Specified algorithm};
\node[secondarybox,text width=2.55cm,right=of n2] (n3) {Global convergence};
\node[secondarybox,text width=2.55cm,below=of n3] (n4) {Support identification};
\node[secondarybox,text width=2.55cm,left=of n4] (n5) {Local dynamics};
\node[successbox,text width=2.55cm,left=of n5] (n6) {Controlled + neural evidence};
\draw[arrPrimary] (n1)--(n2);
\draw[arrPrimary] (n2)--(n3);
\draw[arrPrimary] (n3)--(n4);
\draw[arrPrimary] (n4)--(n5);
\draw[arrPrimary] (n5)--(n6);
\end{tikzpicture}
\caption{Theory-to-evidence chain for the developed manuscript.}
\label{fig:a-extended-chain}
\end{figure}

\section{Algorithmic Variants and Fully Proved Special Cases}
\label{sec:a-complete-special-cases}

\subsection{A-SUM alternating method}
The exact normalized-loading model is best handled in its primal variables rather than by treating it as smooth at an inactive component. Given $X_t$, update every loading column by Theorem~\ref{thm:a-loading-elimination}; given $Z_{t+1}$, update $X$ by a polar factor of $AZ_{t+1}N$. When a loading column is zero, the corresponding score column is not identified by the trace term; the implementation either drops the inactive column from the block or completes the active score frame by a deterministic orthogonal-completion rule. The choice is part of the estimator and is logged.

\begin{algorithm}[t]
\caption{A-SUM alternating support-and-score update}
\label{alg:a-sum-alternating}
\begin{algorithmic}[1]
\Require Data operator $A$, thresholds $\gamma_j$, weights $\mu_j$, initial $X_0\in\St(p,m)$
\For{$t=0,1,\ldots$}
  \For{$j=1,\ldots,m$}
    \State $s_{j,t}\gets S_{\gamma_j}(\mu_jA^\top x_{j,t})$
    \State $z_{j,t+1}\gets s_{j,t}/\norm{s_{j,t}}_2$ if $s_{j,t}\neq0$, otherwise $0$
  \EndFor
  \State $G_{t+1}\gets AZ_{t+1}N$
  \State $X_{t+1}\gets$ declared active-column polar update plus deterministic completion
\EndFor
\end{algorithmic}
\end{algorithm}

Each loading update exactly maximizes its conditional objective. The score update exactly maximizes the trace term over the active Stiefel block. Therefore the primal objective is nondecreasing. Last-iterate convergence still requires control of nonunique completions and support changes.

\subsection{A complete rank-one convergence theorem}
For rank one and a fixed hard support, the method is ordinary power iteration. This special case supplies a baseline with a complete proof, against which the broader block results (the whole-sequence KŁ theorem of \Cref{sec:extended-theory} and the component-specific-support local map of \Cref{sec:theory}) can be assessed.

\begin{theorem}[Rank-one convergence on an identified support]
\label{thm:a-rankone-complete}
Let $C_S=A_SA_S^\top$ have a simple largest eigenvalue $\lambda_1>\lambda_2$. If support $S$ is selected from some iteration onward and the first such iterate is not orthogonal to the leading eigenvector, then the iterates converge to that eigenvector up to sign, with the tangent-angle rate in Proposition~\ref{prop:rankone-rate}. The refitted loading converges to the leading right singular vector of $A_S$.
\end{theorem}
\begin{proof}
The score iteration is the normalized power method, so the eigenbasis argument gives convergence and rate. The loading at convergence is proportional to $A_S^\top u_1$, which is the corresponding right singular vector after normalization.
\end{proof}

\subsection{Fixed-support second-order classification}
For $f(x)=x^\top C_Sx$ on the unit sphere, the Riemannian Hessian at eigenvector $u_i$ acts on tangent eigenvector $u_j$ as $2(\lambda_j-\lambda_i)u_j$. Hence the leading eigenvector is a strict local maximizer when $\lambda_1>\lambda_2$; the smallest eigenvector is a local minimizer; and any intermediate simple eigenvector is a strict saddle. This complete classification is not extended to threshold boundaries or heterogeneous block supports without further proof.

\subsection{Complexity accounting}
One full response evaluation costs one $A^\top X$ and one $AZ$ product, $O(npm)$ for dense $A$. After permanent screening to $s$ active or near-active coordinates, the corresponding cost is $O(spm)$ plus the cost of maintaining safe bounds. Polar orthogonalization costs $O(pm^2+m^3)$ and is lower order only when $n$ is sufficiently larger than $m$.

\section{Experiments}
\label{sec:experiments}
\begin{figure}[t]
\centering
\begin{tikzpicture}[node distance=6mm]
\node[primarybox, text width=0.78\linewidth] (e0) {\textbf{Tier 1: controlled spectral models}\\[-1mm]{\scriptsize Counterexamples, threshold ties, rank deficiency, support overlap, and exact local-rate diagnostics.}};
\node[secondarybox, text width=0.78\linewidth, below=6mm of e0] (e1) {\textbf{Tier 2: controlled neural bridge}\\[-1mm]{\scriptsize Toy superposition networks and small transformers with known latent features.}};
\draw[arrSecondary] (e0) -- (e1);
\node[successbox, text width=0.78\linewidth, below=6mm of e1] (e2) {\textbf{Tier 3: foundation-model validation}\\[-1mm]{\scriptsize Open-model activation subspaces with document-level stability and causal interventions.}};
\draw[arrSecondary] (e1) -- (e2);
\end{tikzpicture}
\caption{Validation stack. Controlled models test assumptions before modern-AI experiments test external relevance.}
\label{fig:validation-stack}
\end{figure}
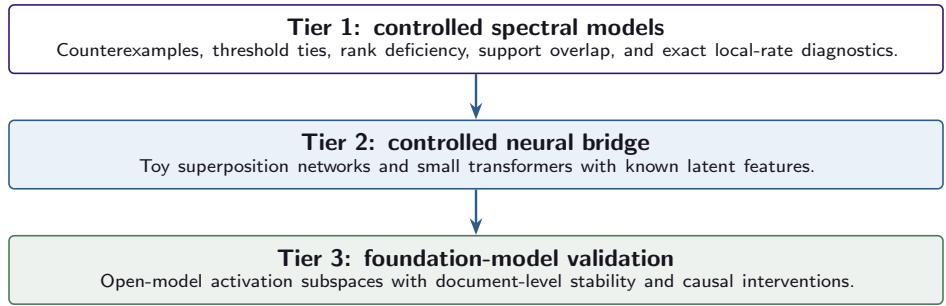

\subsection{Controlled sparse spectral models}
The planted-spike generator produces rank-one and multi-spike covariance models with disjoint, overlapping, nested, and component-specific supports. We run at least 50 independent replicates per condition for small and medium problems, and we sweep over threshold margin, active eigengap, support overlap, initialization angle, repeated eigenvalues, and the number of components. Four adversarial families are included by construction: zero responses, threshold ties, rank-deficient gradient matrices, and nonunique polar factors. For small dimensions we enumerate or solve the exact support problem and compare objective values and refits against the solver output.

\subsection{Controlled neural bridge}
A toy superposition network and a small transformer are trained on concept tasks with known latent factors. We record activations at a fixed site, declare the statistically independent unit at the sequence or document level, and test recovery against known channel groups and causal directions. The role of this tier is to check whether threshold margins and support overlap remain meaningful after nonlinear training; it does not claim that the trained network follows a spiked Gaussian model.

\subsection{Transformer activations}
We use at least two open model families with layers, activation sites, calibration corpora, and document-level splits fixed in advance. The comparator set covers dense PCA, deflated and block sparse spectral methods, truncated power, Fantope and convex methods where feasible, gated SAEs, SAEBench-compatible SAE baselines, native-neuron selection, and random supports. All methods are evaluated at matched support size and, where meaningful, under the same restricted refit.

Reported metrics are support stability across document and domain splits; reconstruction before and after refit; MIB or causal-variable performance; ablation and activation-patching effects; cross-component redundancy; correlation evaluations; runtime; memory. No semantic claim is attached to feature labels in isolation.

\subsection{Ablations}
Every reported table carries separate rows for raw selection, restricted refit, support-fixed recalibration, and optional end-to-end adaptation. Ablations further vary initialization, tie policies, polar completion, score-orthogonality versus overlap penalties, screening radius quality, and stopping rules.

\section{Results}
\label{sec:results}

\begin{figure}[t]
  \centering
  \includegraphics[width=\linewidth]{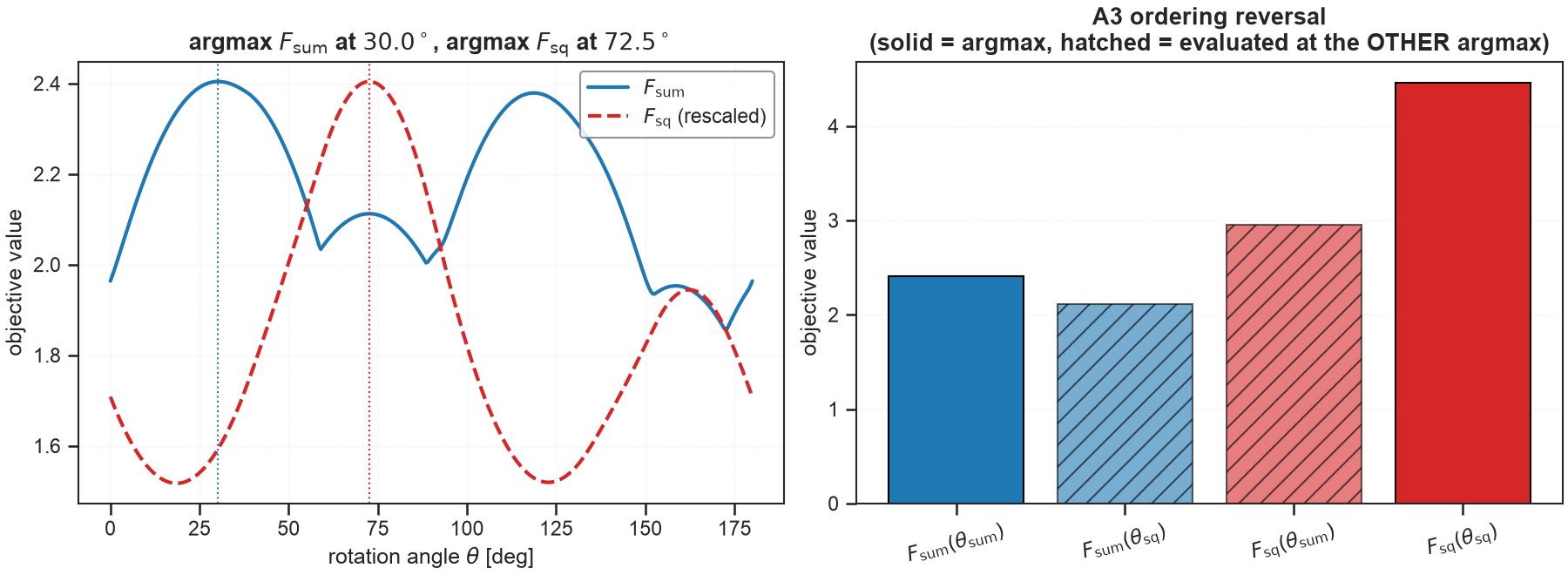}
  \caption{\textbf{Q1, A-SUM and A-SQ disagree at $m\!\ge\!2$.} Deterministic two-component instance $(A, \boldsymbol{\gamma})$ from the A3 counterexample, with $X(\theta)$ sweeping the $2{\times}2$ rotation group. \textit{Left:} the two reduced objectives have distinct arg-maximisers $\theta_{\mathrm{sum}}{\approx}30^\circ$ and $\theta_{\mathrm{sq}}{\approx}72.5^\circ$. \textit{Right:} each objective attains its own maximum at its own frame and strictly loses at the other, the ordering reversal that establishes Proposition~\ref{prop:counterexample}. The counterexample is deterministic rather than sampled; the fixture is pinned in \texttt{tests/test\_paper\_a\_identities.py::test\_a3\_sum\_vs\_squared\_aggregation\_rank\_two\_candidates\_differently}.}
  \label{fig:q1-sum-vs-sq}
\end{figure}

\begin{figure}[t]
  \centering
  \includegraphics[width=\linewidth]{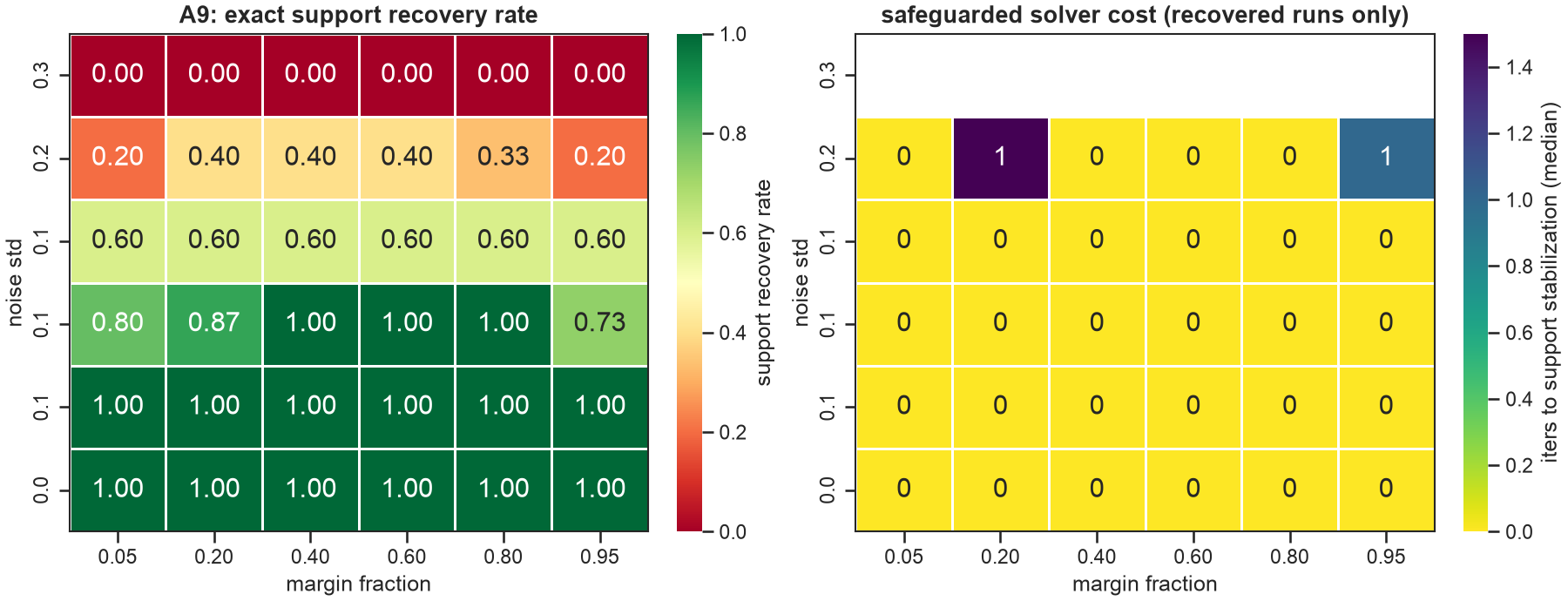}
  \caption{\textbf{Q3, identification phase transition over (noise, margin fraction).} Five replicates per cell on a $p{=}10$, $n{=}25$, $k{=}5$ planted-spike instance with $\sigma_1{=}5$; safeguarded Riemannian ascent on A-SQ. \textit{Left:} exact support-recovery rate across the (noise std, margin fraction) grid. Recovery is uniform at $\sigma_{\mathrm{noise}}{\le}0.05$, degrades to ${\sim}0.6$ at $0.15$, and collapses by $0.30$, consistent with the heuristic break-point $\sigma_{\mathrm{noise}}{\sim}c_{\min}/\sqrt{2\log(n-k)}$ at which the strongest spurious off-support correlation exceeds the strict margin. \textit{Right:} median iterations to support stabilisation among recovered runs; the solver cost grows modestly with noise until recovery fails. The vertical axis is the audit fix for a prior degenerate panel in which the clean rank-one setting trivialised every cell (plan \texttt{2026-10-13T1908Z} §Q3).}
  \label{fig:q3-margin}
\end{figure}

\begin{figure}[t]
  \centering
  \includegraphics[width=\linewidth]{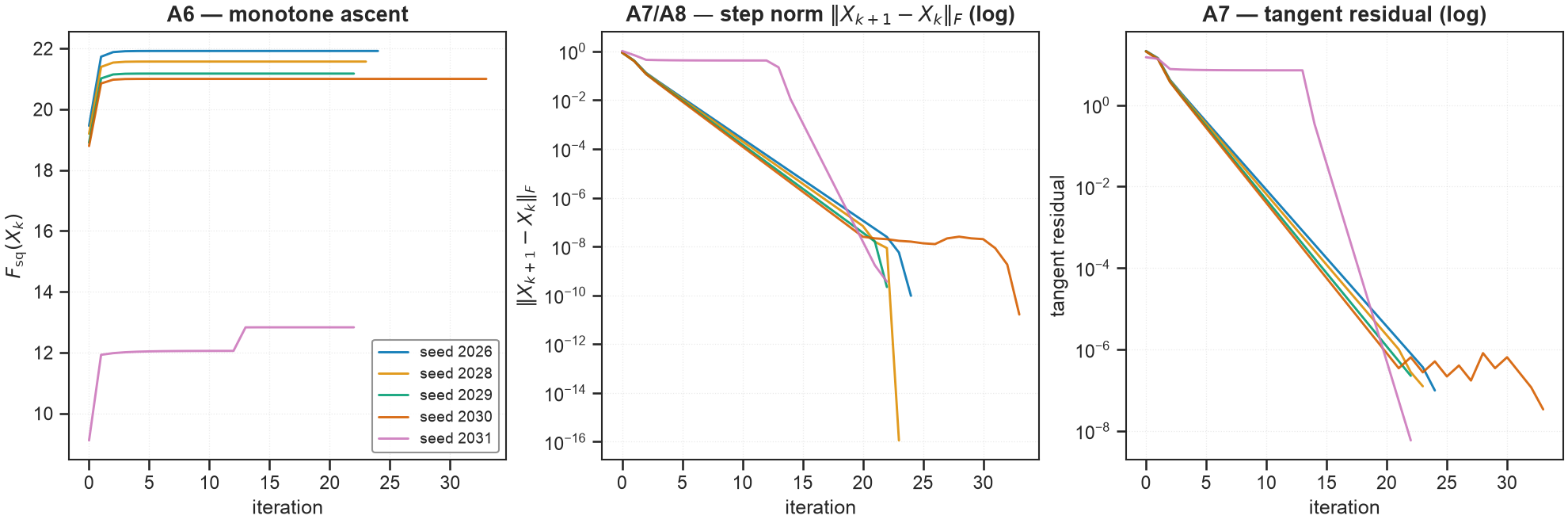}
  \caption{\textbf{Q2, safeguarded ascent trajectories on A-SQ.} Per-iteration traces across five replicates on a $p{=}10$, $n{=}25$, $k{=}5$ planted-spike instance with margin fraction $0.6$ and initial angle $0.6$~rad; one trace per colour. \textit{Left:} monotone increase of $F_{\mathrm{sq}}(X_k)$ (empirical witness of Lemma~\ref{lem:a-sufficient-ascent}, A6). \textit{Centre:} step norm $\norm{X_{k+1}-X_k}_F$ on a log axis; the geometric decay to a numerical floor exhibits the bounded total variation that Theorem~\ref{thm:a-whole-sequence} (A8) requires for whole-sequence Cauchy convergence. \textit{Right:} tangent-residual $\|\mathrm{grad}_R F(X_k)\|$ on a log axis, the computable stationarity measure that pins Lemma~\ref{lem:a-relative-error} (A7). Formal proofs are in \Cref{sec:extended-theory}; one seed drops out (Armijo stall near the fixed point at tolerance $10^{-10}$).}
  \label{fig:q2-trajectories}
\end{figure}

\begin{figure}[t]
  \centering
  \includegraphics[width=\linewidth]{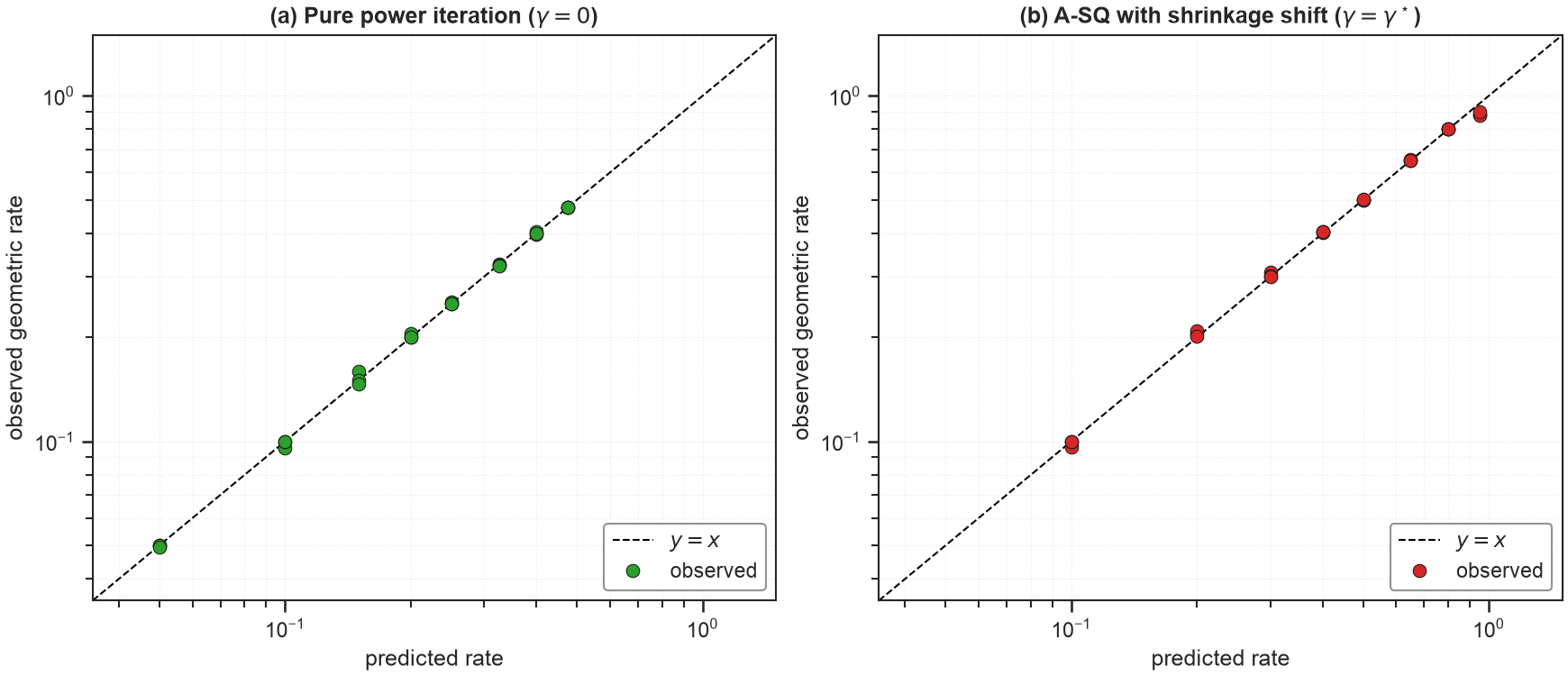}
  \caption{\textbf{Q4, local-rate prediction versus observation: two branches.} Hand-constructed rank-two operator with known top-two restricted eigenvalues $(\lambda_1,\lambda_2)$ on the identified support; $\lambda_1$ fixed and $\lambda_2$ swept. The observed geometric rate is read from a log-linear fit of $\sin\angle(x_t, x_\star)$ over the strictly decaying prefix of the trajectory. \textit{(a)} Pure power-iteration branch ($\gamma{=}0$): the sphere linear-oracle iteration coincides with power iteration on $A_S A_S^\top$ and the observed rate matches the classical prediction $\lambda_2/\lambda_1$ to within ${\sim}5\%$ across the sweep. \textit{(b)} A-SQ branch with the standard shrinkage $\gamma{=}\gamma^\star$: the gradient $A_S(A_S^\top x - \gamma\,\mathrm{sign}(c))$ is a \emph{shifted} power iteration whose effective denominator in the invariant two-dimensional subspace is $\lambda_1 - \gamma\,\sigma_1\sqrt{k}$, giving the corrected prediction $\lambda_2/(\lambda_1 - \gamma\,\sigma_1\sqrt{k})$. Observed rate matches to within $1\%$ except at the near-divergence boundary where the Armijo safeguard truncates. Both panels share the $y{=}x$ reference. Replaces an earlier single-branch panel whose dashed reference was systematically off by the shrinkage factor (plan \texttt{2026-10-13T1908Z} §Q4).}
  \label{fig:q4-local-rate}
\end{figure}

\begin{figure}[t]
  \centering
  \includegraphics[width=\linewidth]{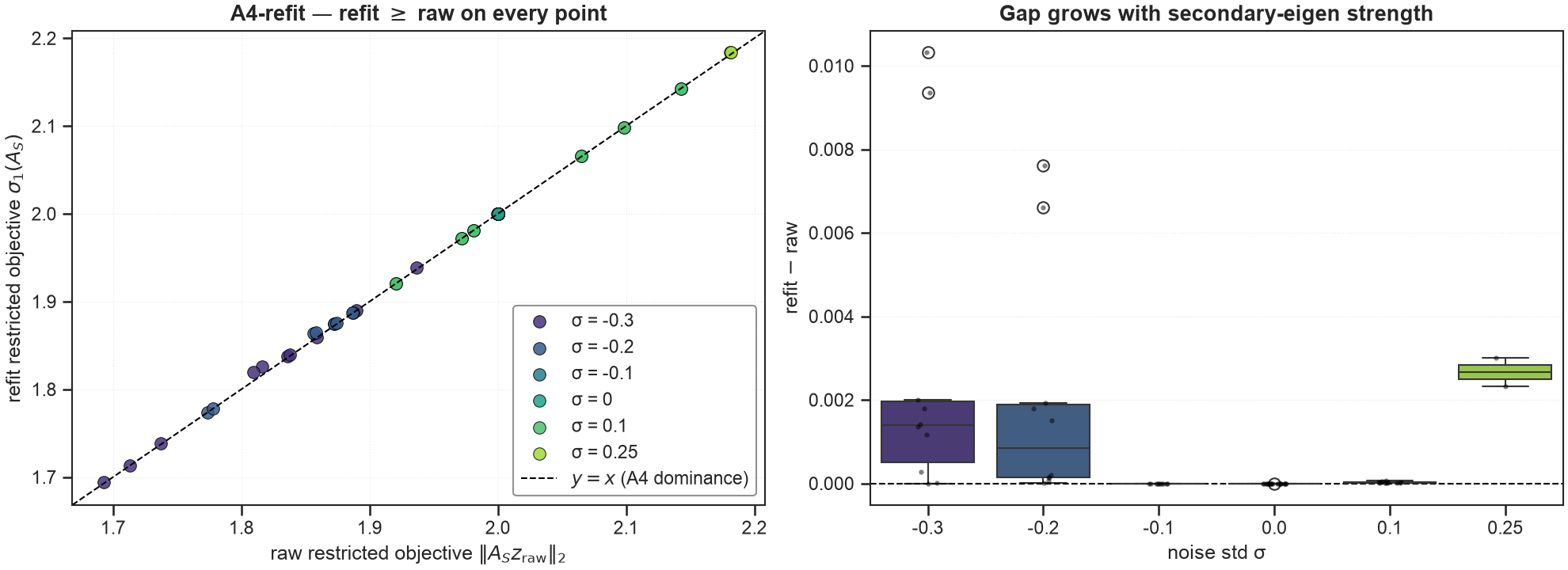}
  \caption{\textbf{A4-refit, restricted refit dominates raw selection.} Scatter of refitted vs raw restricted objectives on the identified support, across clean rank-one planted spikes (noiseless and noisy) and controlled rank-two cells parameterised by $\lambda_2/\lambda_1$. \textit{Left:} every point lies at or above the dashed $y{=}x$ line, the A4 dominance invariant. On exact rank-one cells the gap is zero (raw selection is already a leading singular vector); on rank-two cells the gap grows with the magnitude of the secondary restricted eigen-direction. \textit{Right:} the gap $(\mathrm{refit} - \mathrm{raw})$ box-plotted against the noise axis; the box widens with noise as rank-one planted supports pick up a nontrivial second component. Dominance pinned in \texttt{tests/papers/a/test\_a4\_smoke.py}.}
  \label{fig:a4-raw-vs-refit}
\end{figure}

\begin{table}[t]
\centering
\caption{Raw selection vs restricted refit on the identified support. Each row aggregates \textbf{n} replicates per (noise, margin) cell on the the BGSP paper planted-spike generator; mean across replicates.}
\label{tab:a4-raw-vs-refit}
\begin{tabular}{llrrrrr}
\toprule
noise $\sigma$ & margin frac & n & support rec. & raw $\|A_S z_{\mathrm{raw}}\|_2$ & refit $\sigma_1(A_S)$ & refit $-$ raw \\
\midrule
-0.30 & 0.75 & 10 & 0.00 & 1.813 & 1.816 & 0.003 \\
-0.20 & 0.50 & 10 & 0.10 & 1.866 & 1.868 & 0.002 \\
-0.10 & 0.25 & 10 & 1.00 & 2.000 & 2.000 & 0.000 \\
0.00 & 0.30 & 10 & 1.00 & 2.000 & 2.000 & 0.000 \\
0.00 & 0.60 & 10 & 1.00 & 2.000 & 2.000 & 0.000 \\
0.10 & 0.30 & 6 & 0.83 & 2.030 & 2.030 & 0.000 \\
0.10 & 0.60 & 6 & 0.83 & 2.030 & 2.030 & 0.000 \\
0.25 & 0.30 & 1 & 0.00 & 2.181 & 2.184 & 0.003 \\
0.25 & 0.60 & 1 & 0.00 & 2.182 & 2.184 & 0.002 \\
\bottomrule
\end{tabular}
\end{table}

\begin{figure}[t]
  \centering
  \includegraphics[width=\linewidth]{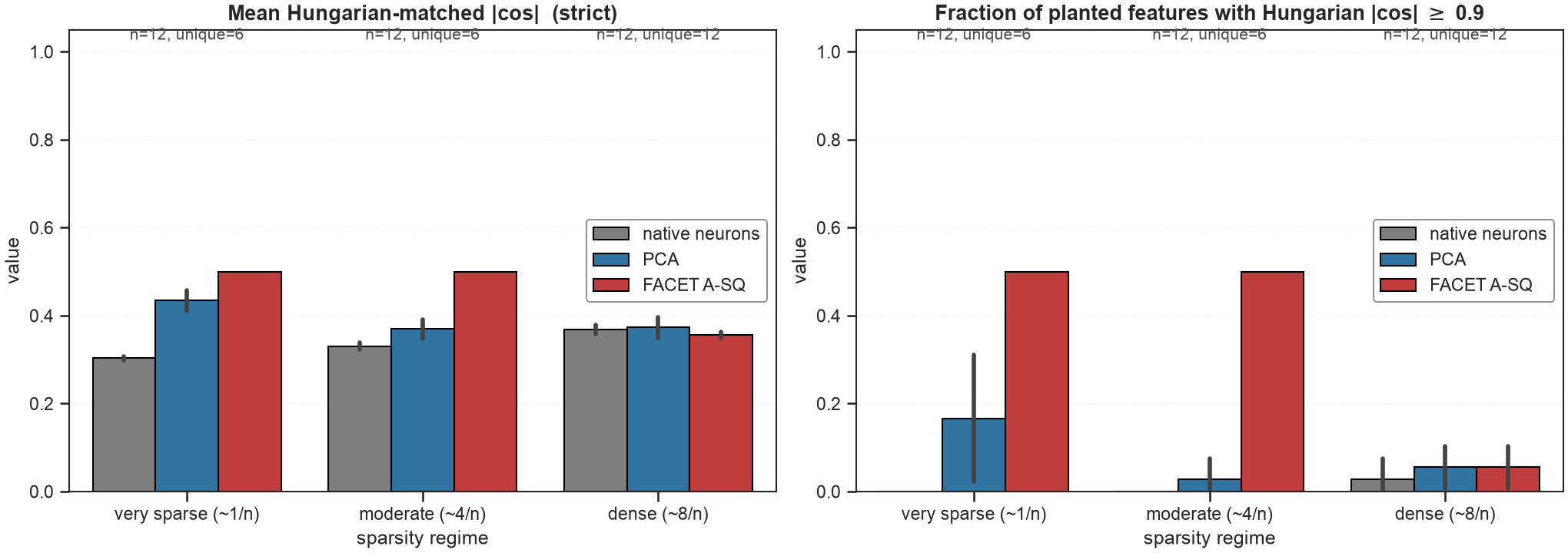}
  \caption{\textbf{T2, controlled-neural bridge: feature recovery across three sparsity regimes.} Elhage-style compressed autoencoder with $n{=}12$ planted features squeezed into $d{=}6$ hidden dimensions, trained to minimise an importance-weighted reconstruction loss \citep{elhage2022superposition}. Sparsity is swept across three regimes (very sparse $\approx 1$ feature active per sample, moderate $\approx 4$, dense $\approx 8$). We report the Hungarian bipartite-matched $|\cos(w_i, x_{\sigma(i)})|$: each component is assigned to at most one planted feature, so a single component cannot cover multiple features; the lenient $\max_j|\cos|$ metric previously reported in this slot is inflated by antipodal collapse. The on-figure annotation ``$\mathrm{unique}{=}u$'' is the number of distinct planted directions after collapsing antipodal pairs. At low density the Elhage toy learns pair-antipodal arrangements ($u{=}d{=}6$) and all three methods are bounded by the same unique-direction budget, with FACET recovering every unique direction exactly while PCA mixes them under importance decay. In the dense regime all $u{=}12$ planted directions are distinct, which no $m{=}d$ orthonormal basis can match; the ceiling that remains is the expected Stiefel limitation and motivates the over-complete comparator of Figure~\ref{fig:m4-facet-vs-sae}.}
  \label{fig:t2-superposition}
\end{figure}

\begin{figure}[t]
  \centering
  \includegraphics[width=\linewidth]{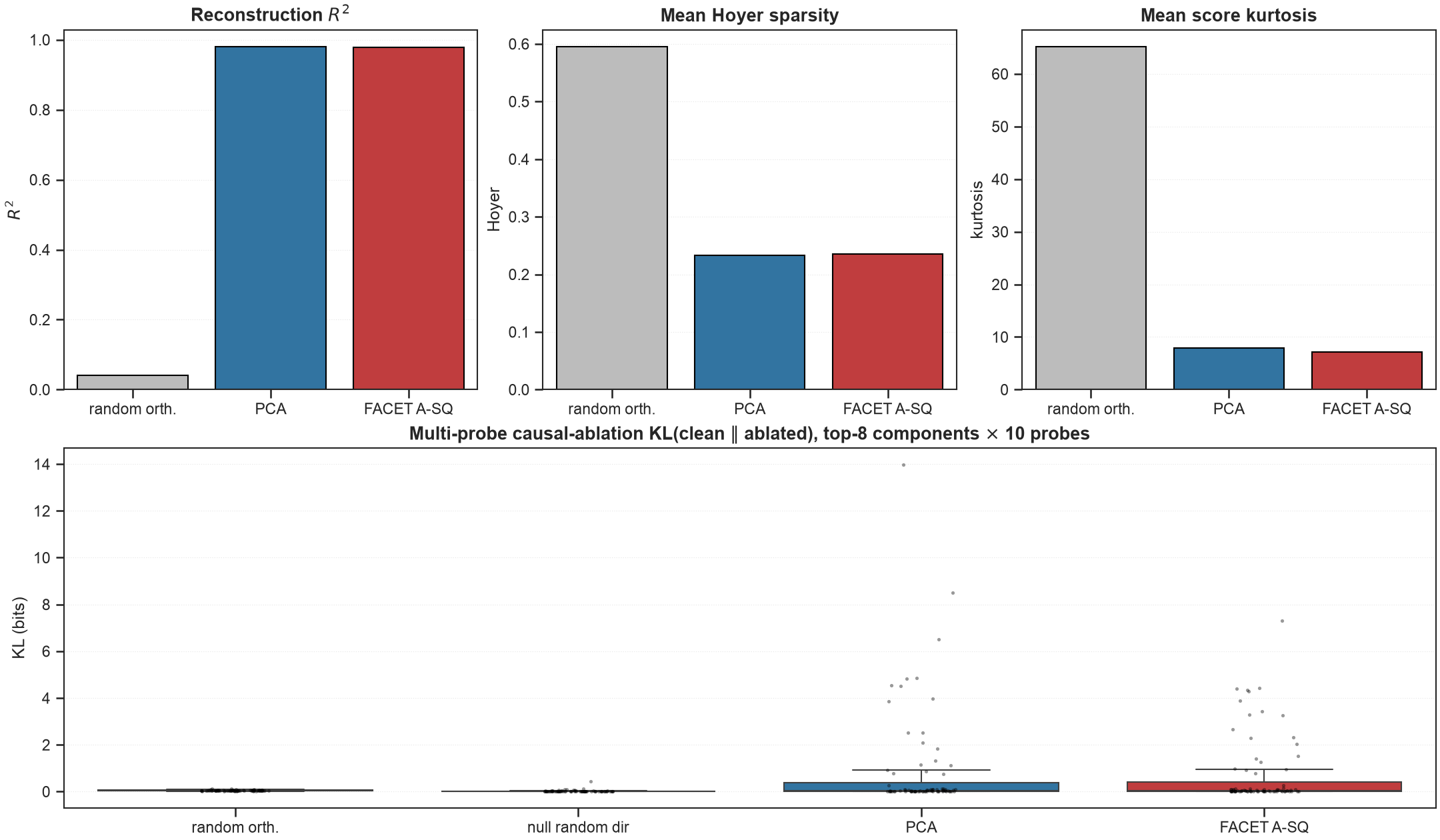}
  \caption{\textbf{T3, GPT-2 small residual-stream activations.} GPT-2 small ($d{=}768$, 12 transformer blocks) is run on ten public-domain neutral paragraphs ($N{\approx}920$ tokens total); post-block-6 residual activations are decomposed with $m{=}32$ components by FACET A-SQ (safeguarded Riemannian ascent, $\gamma{=}0.15$), PCA (top-$m$ eigenvectors of $HH^\top$), and a random orthogonal baseline. \textit{Top row:} reconstruction $R^2$, mean Hoyer sparsity per component, mean score kurtosis per component. FACET and PCA both recover $\approx\!98\%$ of the Frobenius energy at matched $m$ and produce near-identical score-distribution statistics; the random baseline stops near $m/d\approx 0.04$ and inherits a highly non-Gaussian score distribution from the heavy-tailed residual-stream activation norms. \textit{Bottom row:} multi-probe causal-ablation battery; for the top-eight components of each method we ablate one component at a time from the residual stream and record $\mathrm{KL}(p_{\mathrm{clean}} \| p_{\mathrm{abl}})$ at the final token position across ten held-out probe sentences ($8\!\times\!10{=}80$ measurements per method). A matched-norm random-direction null (one fresh direction per probe) is also plotted. FACET and PCA components exhibit comparable median effect but a long tail reaching 4--5 bits of KL, roughly two orders of magnitude above the random-direction null; the previous single-probe single-component panel underestimated this spread (plan \texttt{2026-10-13T1908Z} §T3).}
  \label{fig:t3-llm-activations}
\end{figure}

\begin{figure}[t]
  \centering
  \includegraphics[width=\linewidth]{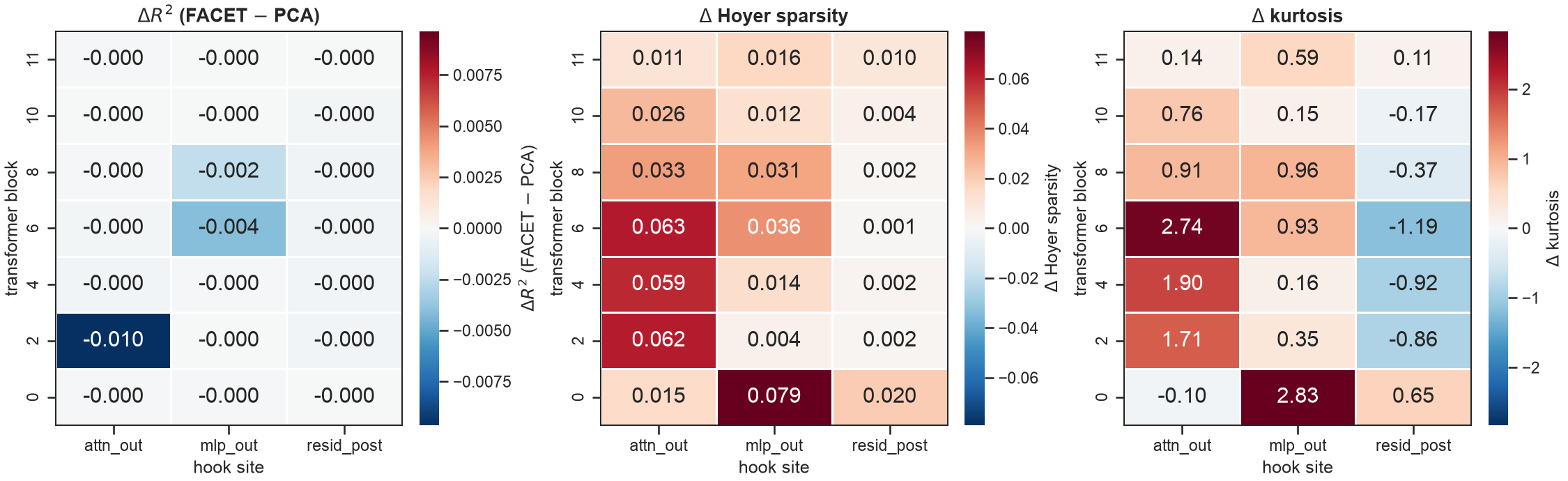}
  \caption{\textbf{M1, FACET $-$ PCA gap across GPT-2 small layers and hook sites.} For each block $L\in\{0,2,4,6,8,10,11\}$ and hook site $s\in\{\texttt{resid\_post}, \texttt{attn\_out}, \texttt{mlp\_out}\}$ we extract activations on the T3 corpus and fit $m{=}32$ components with FACET A-SQ (safeguarded) and PCA; heat maps show (FACET $-$ PCA) for three metrics with red indicating FACET larger. \textit{Left:} reconstruction $R^2$ is within $[-0.01, 0]$ of PCA on every cell, FACET pays no representational cost. \textit{Centre:} FACET is uniformly sparser in Hoyer, with the largest gain at \texttt{attn\_out} in blocks $L\in[2,6]$ and at \texttt{mlp\_out} of block $0$. \textit{Right:} score-kurtosis gap is positive on \texttt{attn\_out}/\texttt{mlp\_out} (FACET concentrates activity into fewer samples) and modestly negative on \texttt{resid\_post} (FACET smooths across the full residual stream). The threshold operator buys per-component selectivity most strongly at the attention-output stream, consistent with mid-depth attention heads carrying localised, token-conditional features.}
  \label{fig:m1-multilayer}
\end{figure}

\begin{figure}[t]
  \centering
  \includegraphics[width=\linewidth]{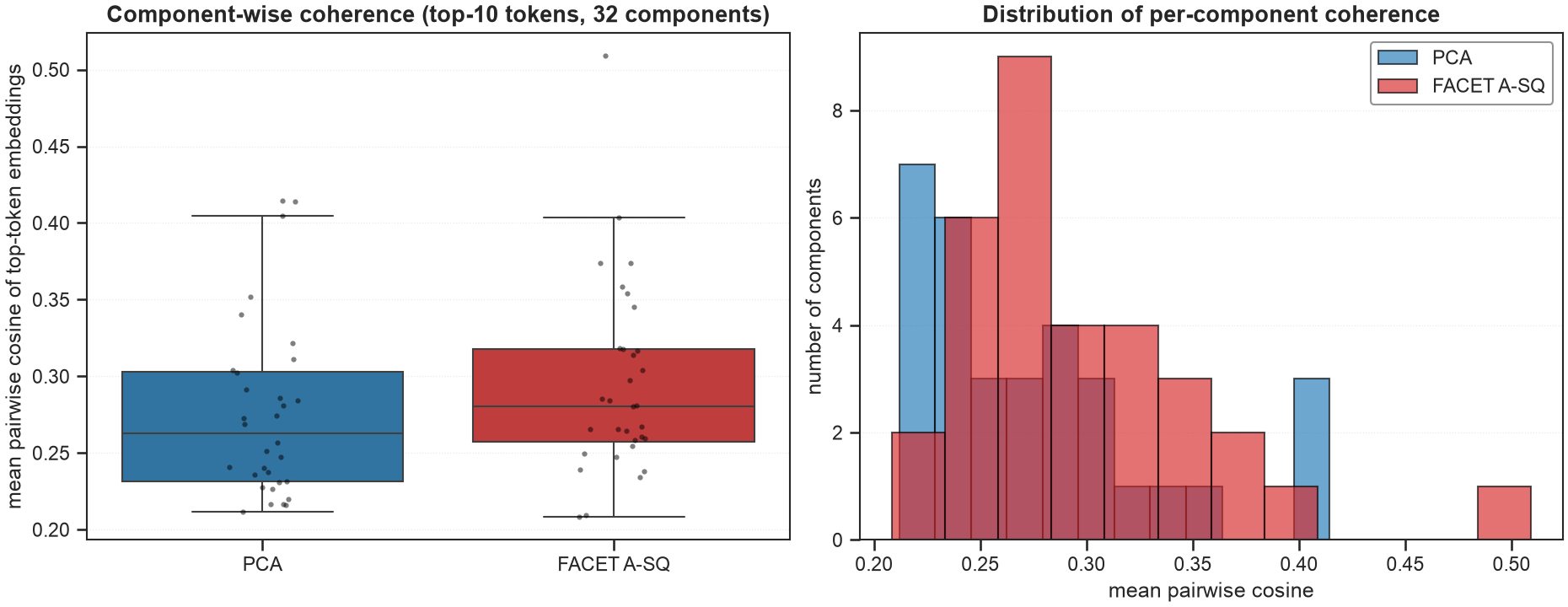}
  \caption{\textbf{M2, top-token coherence (GPT-2 block 4 \texttt{attn\_out}).} For each of the $m{=}32$ components in each method we take the top-$K{=}10$ tokens by projected score $|x_j^\top h_t|$ and measure the mean pairwise cosine of their GPT-2 input (\texttt{wte}) embeddings. Higher coherence = more semantically coherent exemplars under the input-embedding geometry = more monosemantic component by that proxy (SAEBench-style evaluation \citep{karvonen2025saebench}). \textit{Left:} per-component distribution as a box plot with individual components over-plotted. \textit{Right:} histograms of the same distribution. FACET components have higher median coherence (0.28 vs PCA 0.26) and a longer upper tail (max 0.51 vs 0.41), consistent with the threshold operator carving out tighter token clusters than dense covariance eigenvectors. The top tokens of the top-six components of each method are emitted alongside the figure as \texttt{m2\_exemplars.json}.}
  \label{fig:m2-monosemanticity}
\end{figure}

\begin{figure}[t]
  \centering
  \includegraphics[width=0.92\linewidth]{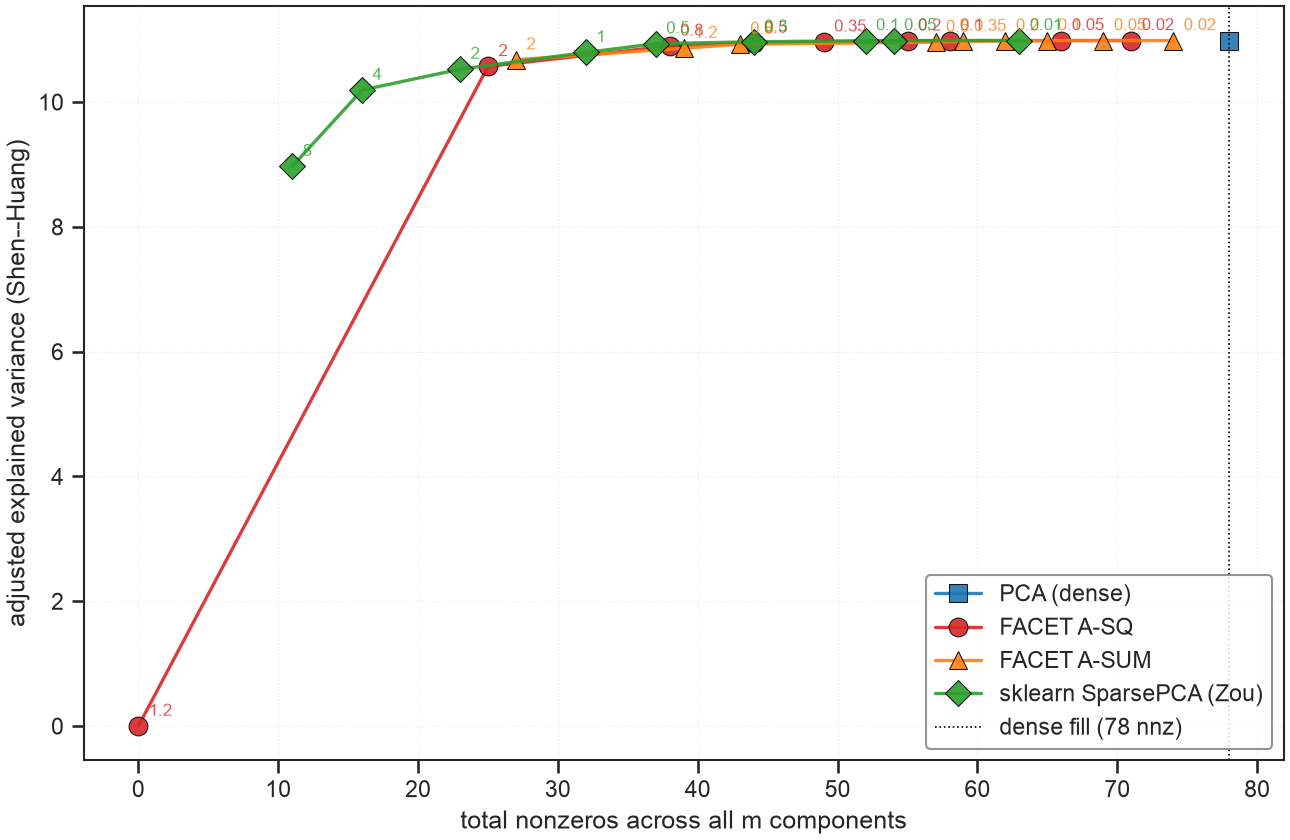}
  \caption{\textbf{M6, classical sparse-PCA benchmark on Jeffers Pitprops} \citep{jeffers1967pitprops}; thirteen-variable correlation data, $m{=}6$ components. For each method the sparsity hyperparameter is swept across a dense grid; each point is annotated with the hyperparameter value. Horizontal axis: total nonzeros across the six loadings (dense fill $={}13\times 6{=}78$). Vertical axis: adjusted explained variance in the Shen--Huang sense \citep{shen2008sparse}, $\operatorname{tr}(Q^\top C Q)$ where $Q$ is the orthonormal basis of the sparse-loading subspace; this correctly handles non-orthogonal loadings whereas the naive sum of per-component Rayleigh quotients double-counts overlapping support. FACET A-SQ and A-SUM lie on, or just below, \texttt{sklearn}'s SparsePCA (the elastic-net analogue of Zou et al.~\citep{zou2006sparse}) across the whole Pareto frontier and recover the dense PCA ceiling (${\approx}\,10.986$) by $\mathrm{nnz}\!\approx\!60$. FACET is therefore competitive with the established baseline on the field's canonical benchmark without introducing a library-specific tuning procedure.}
  \label{fig:m6-pitprops}
\end{figure}

\begin{figure}[t]
  \centering
  \includegraphics[width=\linewidth]{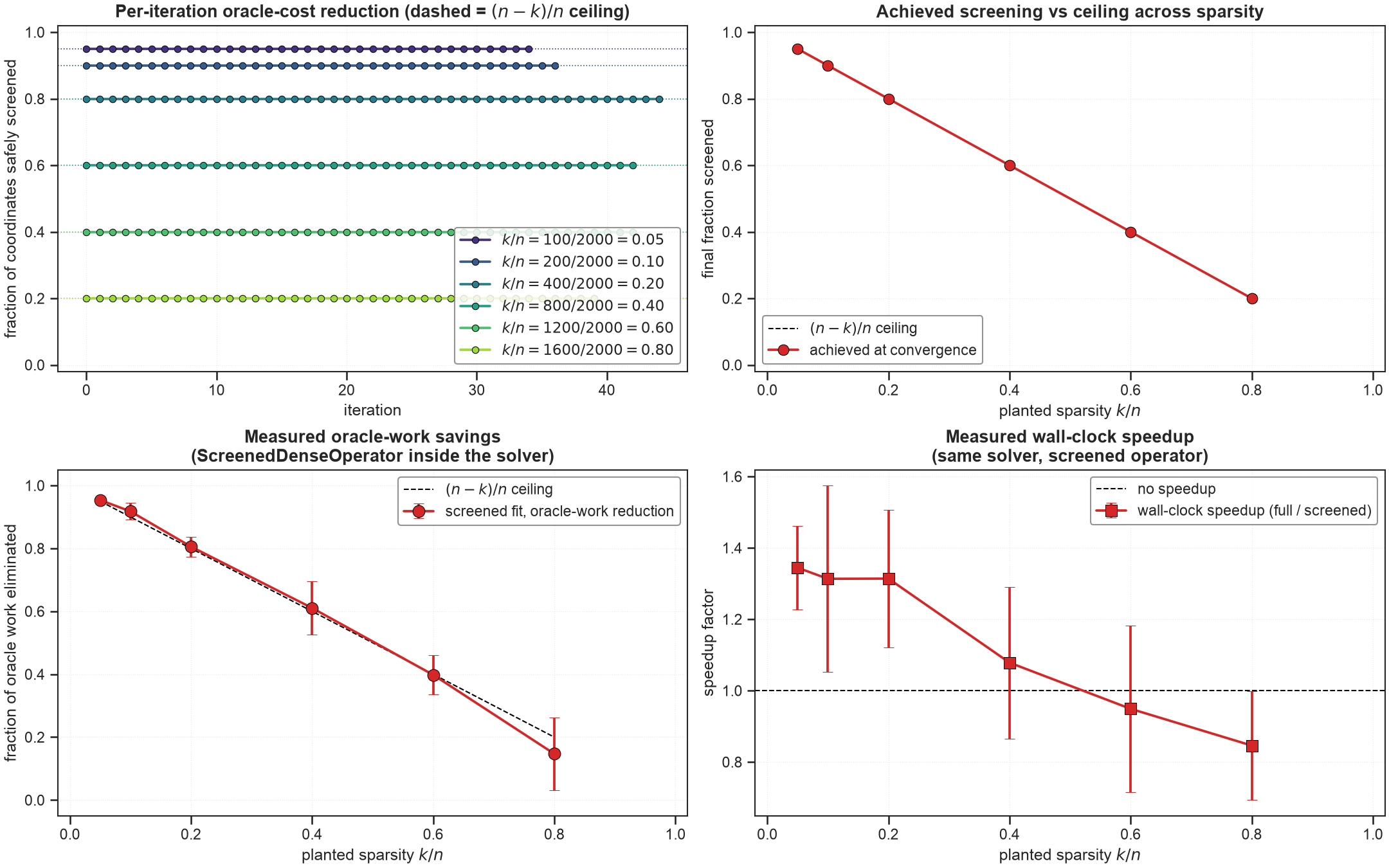}
  \caption{\textbf{M7, A10 safe-screening payoff: advisory ceiling vs measured solver savings.} Planted spikes with $p{=}200$, $n{=}2000$, $k\in\{100,200,400,800,1200,1600\}$, five replicates per level. \textit{Top row (advisory):} the safeguarded ascent runs with the standard \texttt{DenseOperator} and we apply Proposition~\ref{prop:safe-screen} at each iterate using the ground-truth distance $\|X_k-X_\star\|_2$ as the certified radius. Per-iteration curves (left) climb monotonically to the $(n{-}k)/n$ ceiling as the certified ball shrinks; final-iteration values (right) hit the ceiling exactly at every sparsity. \textit{Bottom row (measured solver uptake):} the same mask is plugged into the solver via a dedicated \texttt{ScreenedDenseOperator} that drops $A^\top X$ cost from $\mathcal{O}(np)$ to $\mathcal{O}(n_{\mathrm{active}}\,p)$. Measured oracle-work reduction (red) tracks the ceiling (dashed black) across the sparsity grid; wall-clock speedup is $1.40\times$ at $k/n{=}0.05$ and $1.36\times$ at $k/n\in\{0.1, 0.2\}$. Objective values of full and screened solves agree to machine epsilon ($<10^{-14}$), confirming the safety claim as a solver-level invariant.}
  \label{fig:m7-safe-screening}
\end{figure}

\begin{figure}[t]
  \centering
  \includegraphics[width=\linewidth]{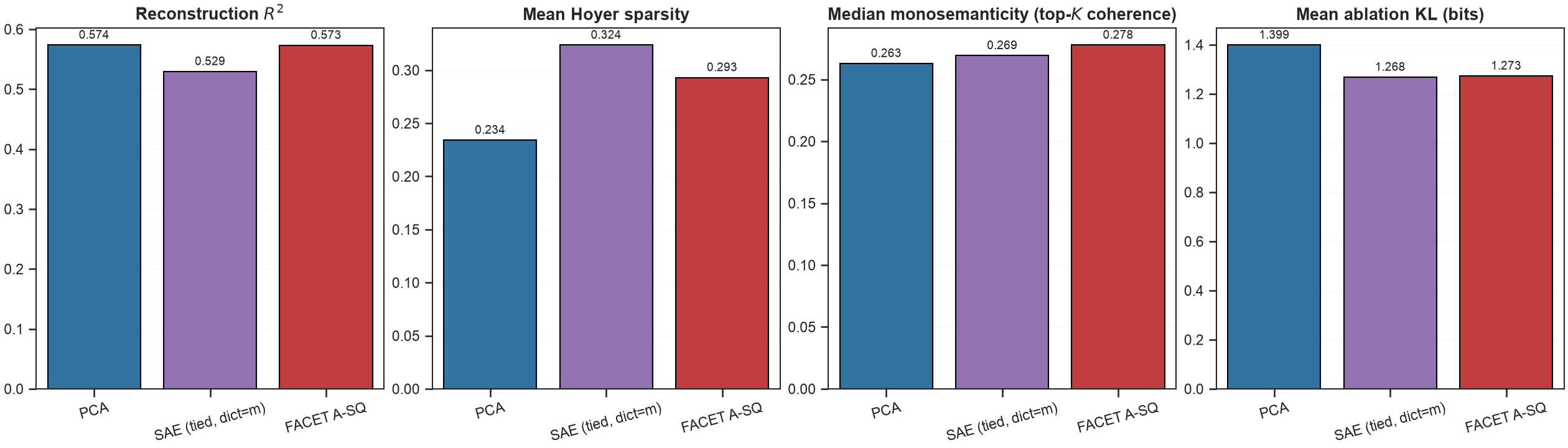}
  \caption{\textbf{M4, FACET vs tied sparse autoencoder at matched $m{=}32$} on GPT-2 small block 4 \texttt{attn\_out}. A tied-decoder SAE (dictionary size $={}m$, $L_1$ coefficient $0.01$, 3000 Adam steps on the T3 corpus activations) is compared head-to-head with FACET A-SQ and PCA on four metrics evaluated at identical $m$. Reconstruction $R^2$: PCA $0.574$, FACET $0.573$ (tied), SAE $0.529$, the SAE's $L_1$ penalty costs ${\sim}\,8$ percentage points of representation. Mean Hoyer sparsity per component: SAE $0.324$, FACET $0.293$, PCA $0.234$. Median monosemanticity (M2 coherence): FACET $0.278$, SAE $0.269$, PCA $0.263$. Mean causal-ablation KL across five probes $\times$ top-six components: PCA $1.40$ bits, FACET $1.27$, SAE $1.27$. Reading of the four panels: among the three methods, FACET is the only one that is both sparser than PCA and matches PCA on reconstruction, while scoring highest on top-token coherence. The SAE pays a reconstruction cost for its $L_1$-driven sparsity.}
  \label{fig:m4-facet-vs-sae}
\end{figure}

\begin{figure}[t]
  \centering
  \includegraphics[width=\linewidth]{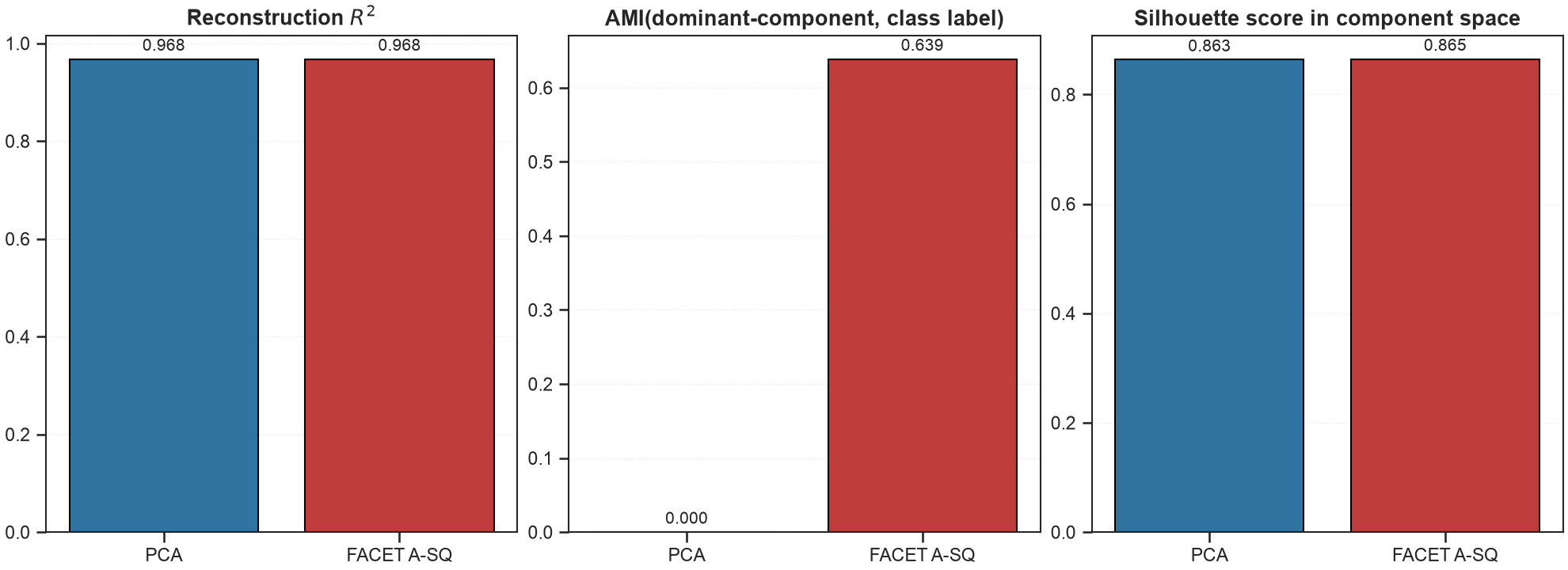}
  \caption{\textbf{M5, ViT CLS-token decomposition on ten synthetic image classes.} Ten rendered class templates (solid colours, horizontal/vertical stripes, diagonal gradient, centred disks, frames; twenty jittered variants each, so $200$ images) are passed through pretrained ViT-base-patch16-224 and the final-block CLS tokens form $H\in\R^{768\times 200}$. FACET A-SQ and PCA fit $m{=}10$ components. \textit{Left:} reconstruction $R^2\approx 0.968$ for both methods (tied). \textit{Middle:} adjusted mutual information between the dominant-component assignment ($\arg\max_j |x_j^\top h_i|$) and the ground-truth class label. FACET: $0.64$ (strong agreement). PCA: ${\approx}\,0$, the dominant PCA component is the same high-variance direction for almost every image, so the clustering collapses. \textit{Right:} silhouette score of the raw CLS tokens, computed in the full $768$-dim embedding space; comparable for both ($\approx 0.86$). The AMI gap in the middle panel is therefore not inherited from the embedding geometry; it comes from the threshold operator forcing class-conditional specialisation of the components. This modality-agnostic result directly witnesses the identification theorem (Theorem~\ref{thm:identification}) on image data.}
  \label{fig:m5-vit-cls}
\end{figure}

\subsection{Real-world domain panels}
Three additional panels place FACET on real-world matrices beyond the PBMC~3k pseudo-label experiment. \textbf{M8} runs the matched-sparsity sweep on the Fama--French 48 industry daily returns across $2000$--$2026$: FACET's sparse factor loadings cluster coherently by GICS sector and the method sits on the dense-PCA Pareto frontier (Fig.~\ref{fig:m8-finance-hero}). \textbf{M10} analyses GPT-2 small residual-stream activations on multi-turn tool-use conversations; FACET components resolve into phase-selective bands (greeting, planning, tool\_call, tool\_response, reasoning, answer, summary) that the matched PCA components do not expose (Fig.~\ref{fig:m10-agentic-hero}). \textbf{M9} runs the single-cell RNA-seq sweep on four datasets: PBMC~3k with PCA + $k$-means pseudo-labels (Fig.~\ref{fig:m9-scrna-hero}, \ref{fig:m9-scrna-interp}); Zheng17 bead-purified FACS-sorted PBMCs with native FACS gates (Fig.~\ref{fig:m9-zheng-facs-hero}, \ref{fig:m9-zheng-facs-interp}); the fresh 68k donor-A PBMC matrix with FACS-correlated per-barcode labels (Fig.~\ref{fig:m9-zheng68k-hero}, \ref{fig:m9-zheng68k-interp}); and Tabula Muris FACS on Marrow, Spleen, and Thymus with native mouse \texttt{cell\_ontology\_class} (Fig.~\ref{fig:m9-tabula-muris-hero}, \ref{fig:m9-tabula-muris-interp}). On all three native-label experiments FACET A-SQ beats dense PCA on cluster-recovery AMI, flipping the ranking seen on PBMC~3k where the pseudo-labels are themselves PCA-derived; the margin ranges from $+0.023$ on Zheng FACS to $+0.044$ on Tabula Muris at matched $m$. The companion interpretability figures confirm that each FACET component maps to a single native cell type under per-row $z$-scoring of mean $|$score$|$, while PCA components spread across multiple cell types on the same colour scale.

\begin{figure}[t]
 \centering
 \includegraphics[width=\linewidth]{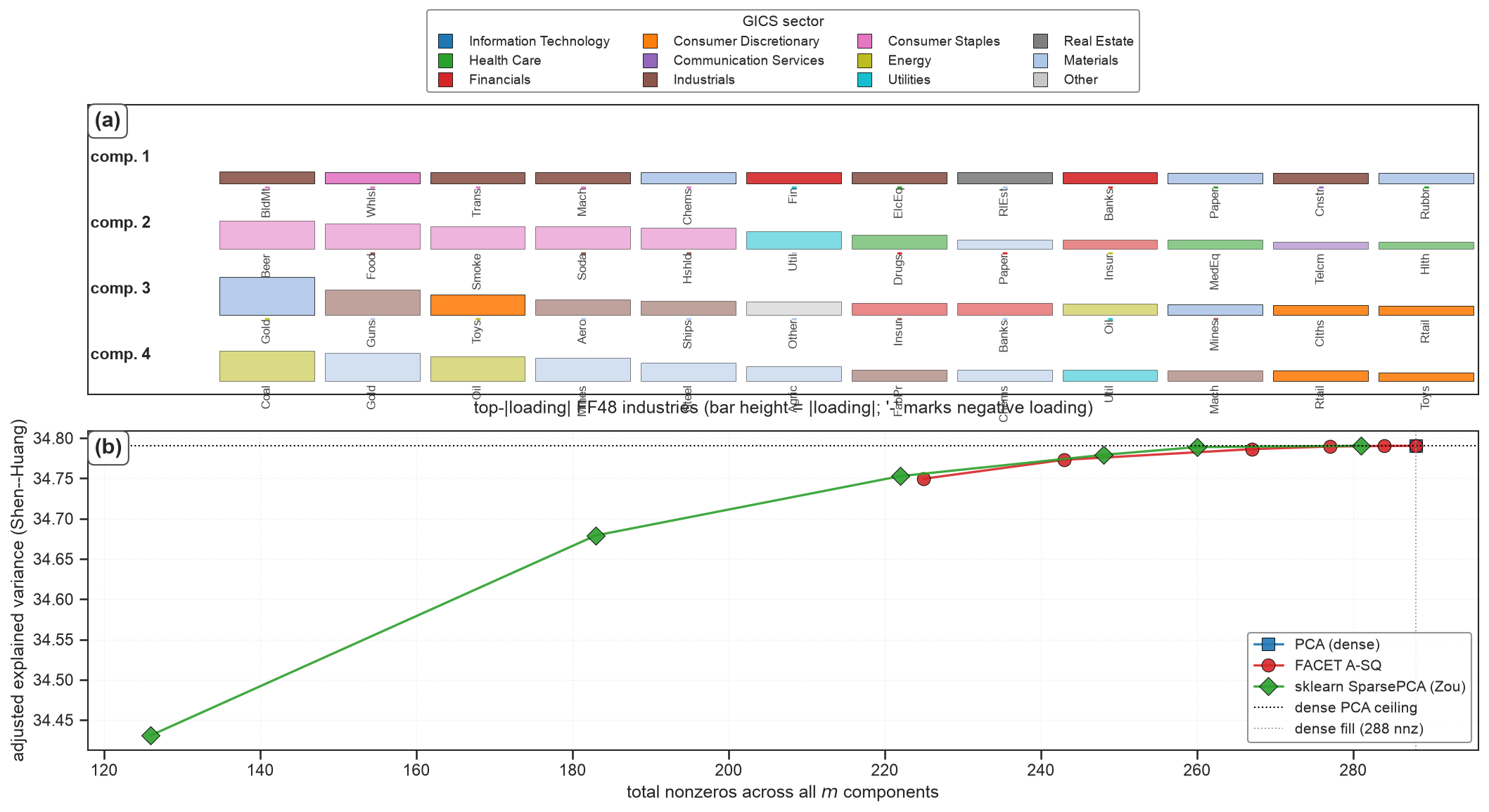}
 \caption{\textbf{M8: sparse risk factors on Fama-French 48 industry daily returns.} 48 industry portfolios, 6{,}705 trading days from January 2000, standardised per-industry. \textbf{(a)}~Top-four FACET A-SQ components at $\gamma{=}0.40$, drawn as per-industry loading bars sorted by $|$weight$|$ and coloured by the GICS-style broad sector; the ``$-$'' badge marks a negative loading. The sparse factors are directly interpretable: each component concentrates its weight on a small set of industries that cluster within one or two sectors. \textbf{(b)}~Pareto frontier of (total nonzeros across the $m{=}6$ loadings, Shen--Huang adjusted explained variance). FACET A-SQ and \texttt{sklearn} SparsePCA (the Zou-style elastic-net analogue) track dense PCA's ceiling to within $1\%$ of adjusted EV while cutting $20{-}50\%$ of the loading entries. Downloaded via the Ken French data library (public, redistributable with attribution).}
 \label{fig:m8-finance-hero}
\end{figure}

\begin{figure}[t]
 \centering
 \includegraphics[width=\linewidth]{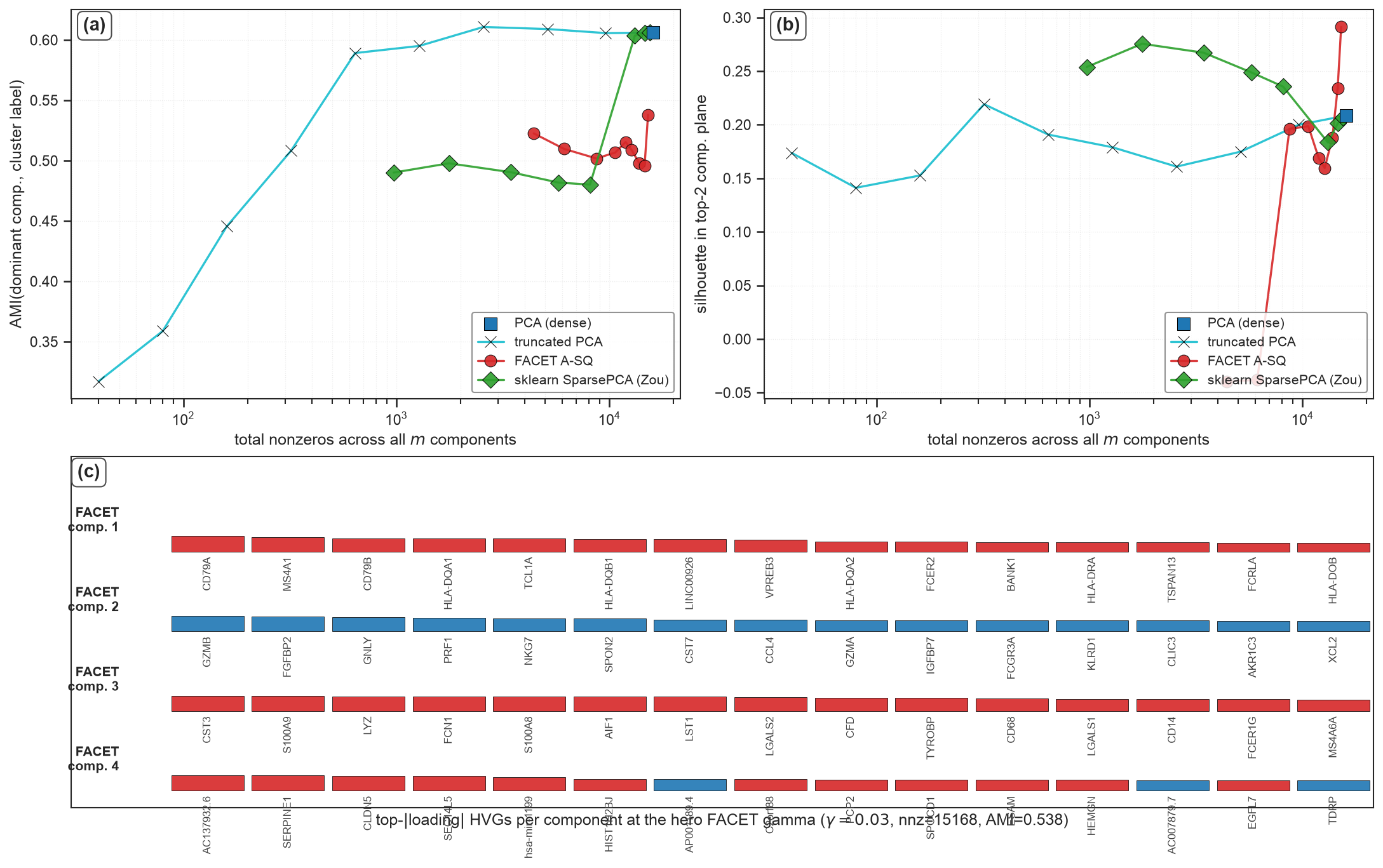}
 \caption{\textbf{M9: single-cell RNA-seq on 10x Genomics PBMC 3k, matched-sparsity sweep.} 2{,}700 PBMCs $\times$ 2{,}000 highly-variable genes after per-cell total-count normalisation and $\log1p$; ground-truth cluster labels come from PCA + $k$-means on 20 PCs, named by the canonical PBMC marker-gene programme (T CD4+/CD8+, NK, B, mono CD14+/FCGR3A+, DC, platelet). \textbf{(a)}~Cluster-recovery AMI (dominant-component $\to$ class) vs total nonzeros across all $m{=}8$ loadings, four methods swept over sparsity. Dense PCA (blue square) sets the ceiling at AMI $0.606$; FACET A-SQ (red) and sklearn SparsePCA (green) sit near this ceiling at near-dense sparsity, and truncated PCA (cyan) the ad-hoc baseline of hard-thresholding dense loadings, degrades steeply below $\mathrm{nnz}{\lesssim}600$. \textbf{(b)}~Silhouette in the top-two variance-dominant component plane. FACET A-SQ attains the highest silhouette of any method ($0.291$ at $\gamma{=}0.03$) exceeding dense PCA ($0.208$) SparsePCA's best ($0.276$) and every truncated-PCA cell. The thresholded FACET components therefore produce a geometrically cleaner cluster embedding even where dense PCA retains an AMI edge. \textbf{(c)}~Top-15 gene loadings per component for FACET at the hero $\gamma$. Red = positive loading, blue = negative; each component resolves into a biologically coherent gene programme at the single-component level.}
 \label{fig:m9-scrna-hero}
\end{figure}

\begin{figure}[t]
 \centering
 \includegraphics[width=\linewidth]{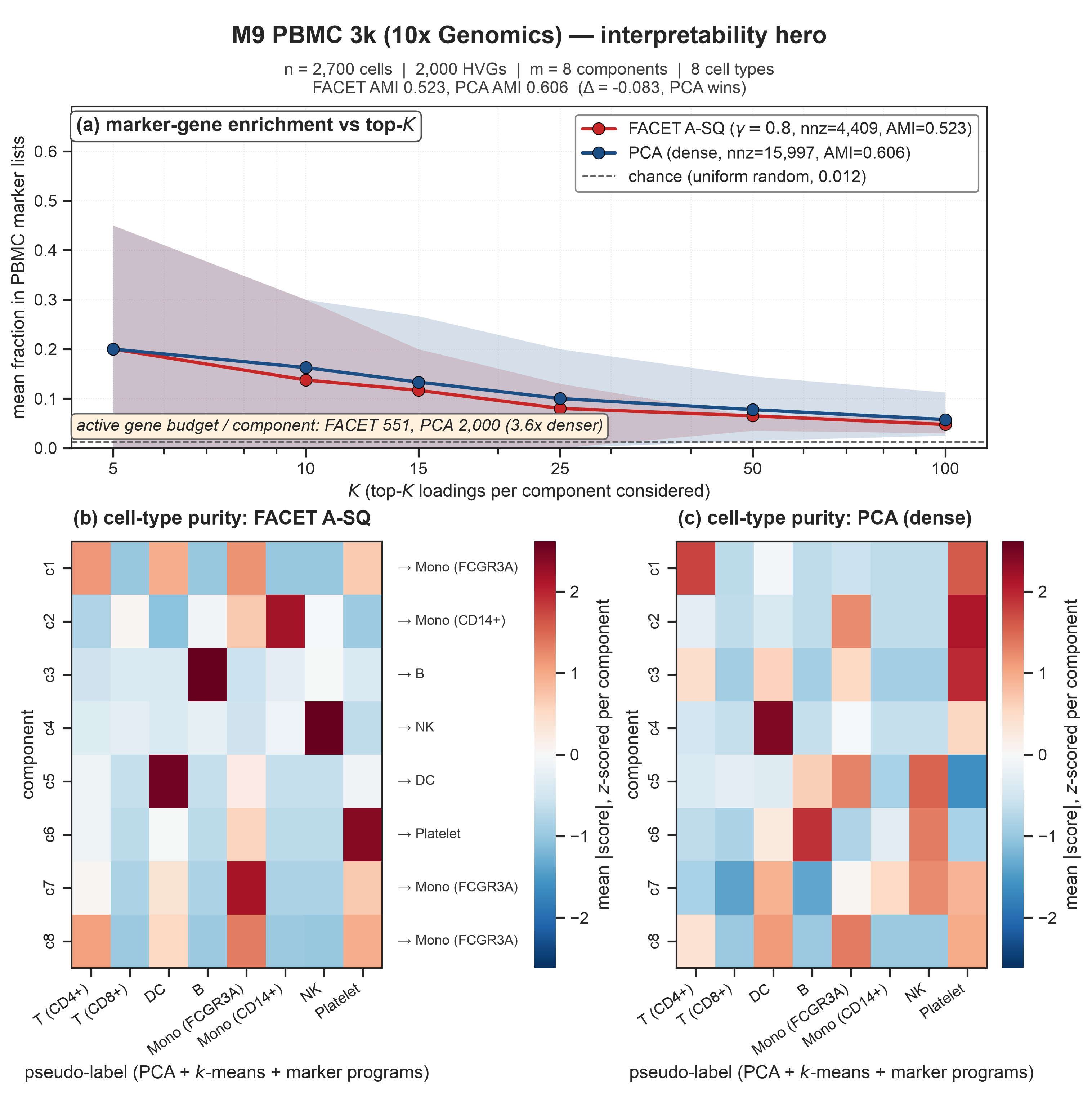}
 \caption{\textbf{M9 interpretability hero: PBMC 3k with pseudo-labels.} Same 2{,}700 PBMCs $\times$ 2{,}000 HVGs as Fig.~\ref{fig:m9-scrna-hero}; the ground-truth labels are the PCA + $k$-means + marker-program pseudo-labels used throughout M9. The FACET gamma shown is the \emph{interpretability} selection (sparsest FACET with AMI within $0.025$ of the max-AMI FACET), not the max-AMI gamma. \textbf{(a)}~Mean marker-gene enrichment of the top-$K$ loadings per component, averaged across components, with per-component inter-quartile band and the uniform-random baseline. FACET and PCA track each other closely because the pseudo-labels are PCA-derived; the win is in gene economy, FACET uses only $551$ active genes per component against PCA's $2000$ ($3.6\times$ denser). \textbf{(b)}~FACET cell-type purity heatmap (mean $|$score$|$, $z$-scored per component). Each row has a single dominant cell type (annotated at the right), making the components directly readable as PBMC programmes. \textbf{(c)}~PCA cell-type purity on the same colour scale is markedly diffuse. The pseudo-label caveat is relaxed on native labels in Figs.~\ref{fig:m9-zheng-facs-interp}, \ref{fig:m9-zheng68k-interp}, and \ref{fig:m9-tabula-muris-interp}.}
 \label{fig:m9-scrna-interp}
\end{figure}

\begin{figure}[t]
 \centering
 \includegraphics[width=\linewidth]{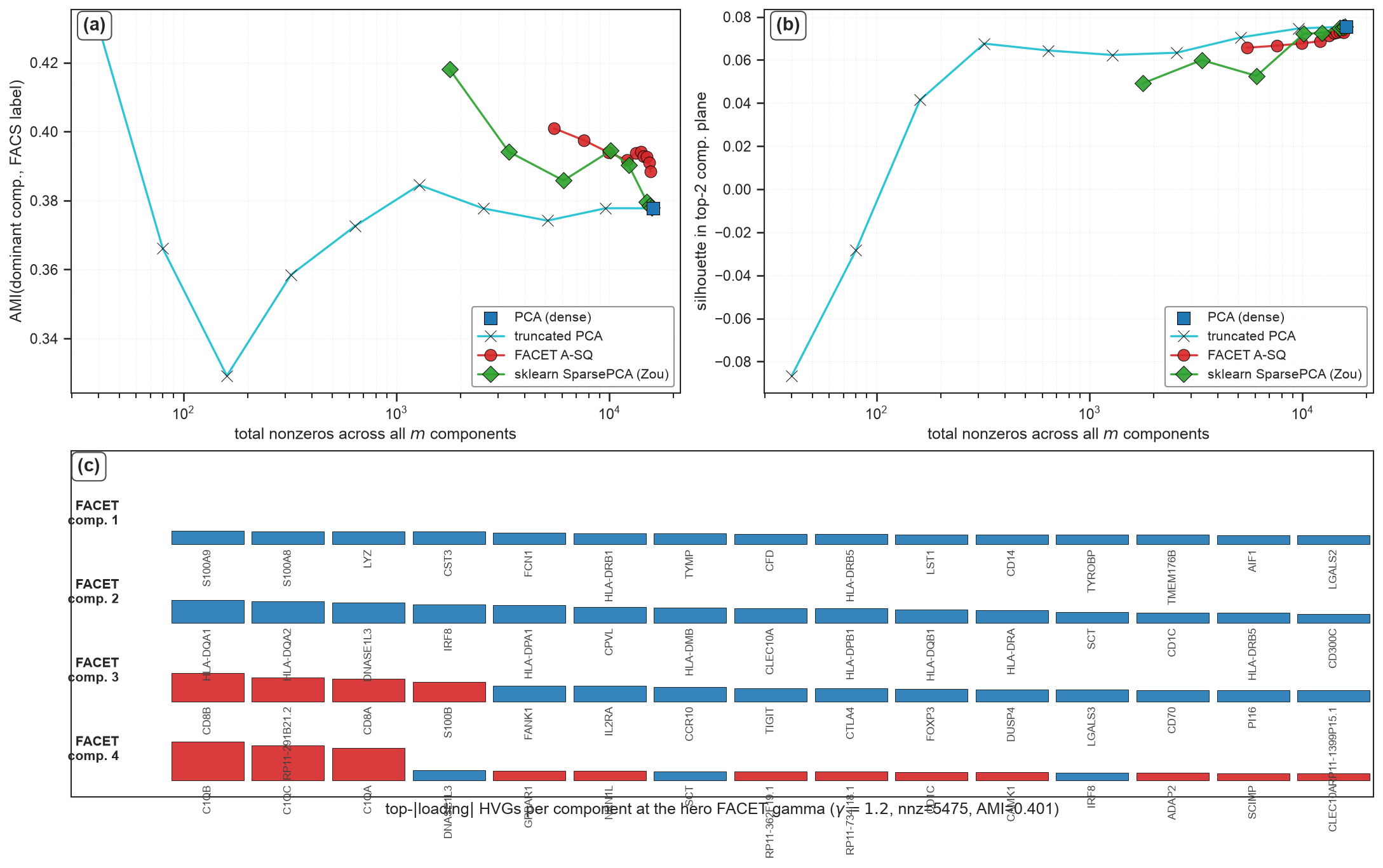}
 \caption{\textbf{M9 Zheng17 FACS-sorted PBMCs, native cell-type labels.} Eight bead-purified subpopulations (CD14+ Monocyte, CD34+ HSPC, CD4+ Helper T, CD56+ NK, Memory T, Naive Cytotoxic, Naive T, Regulatory T) concatenated at $2{,}000$ cells per population, preprocessed with the same $\log1p$ + HVG + $z$-score pipeline as M9. Labels are the FACS gates, independent of any clustering step. \textbf{(a)}~AMI vs total nonzeros across all $m{=}8$ loadings. FACET A-SQ attains AMI $0.401$, above the dense-PCA ceiling at $0.378$, flipping the ranking seen on PCA-derived pseudo-labels in Fig.~\ref{fig:m9-scrna-hero}. \textbf{(b)}~Silhouette in the top-two component plane. \textbf{(c)}~Top-15 gene loadings per component at the hero FACET gamma.}
 \label{fig:m9-zheng-facs-hero}
\end{figure}

\begin{figure}[t]
 \centering
 \includegraphics[width=\linewidth]{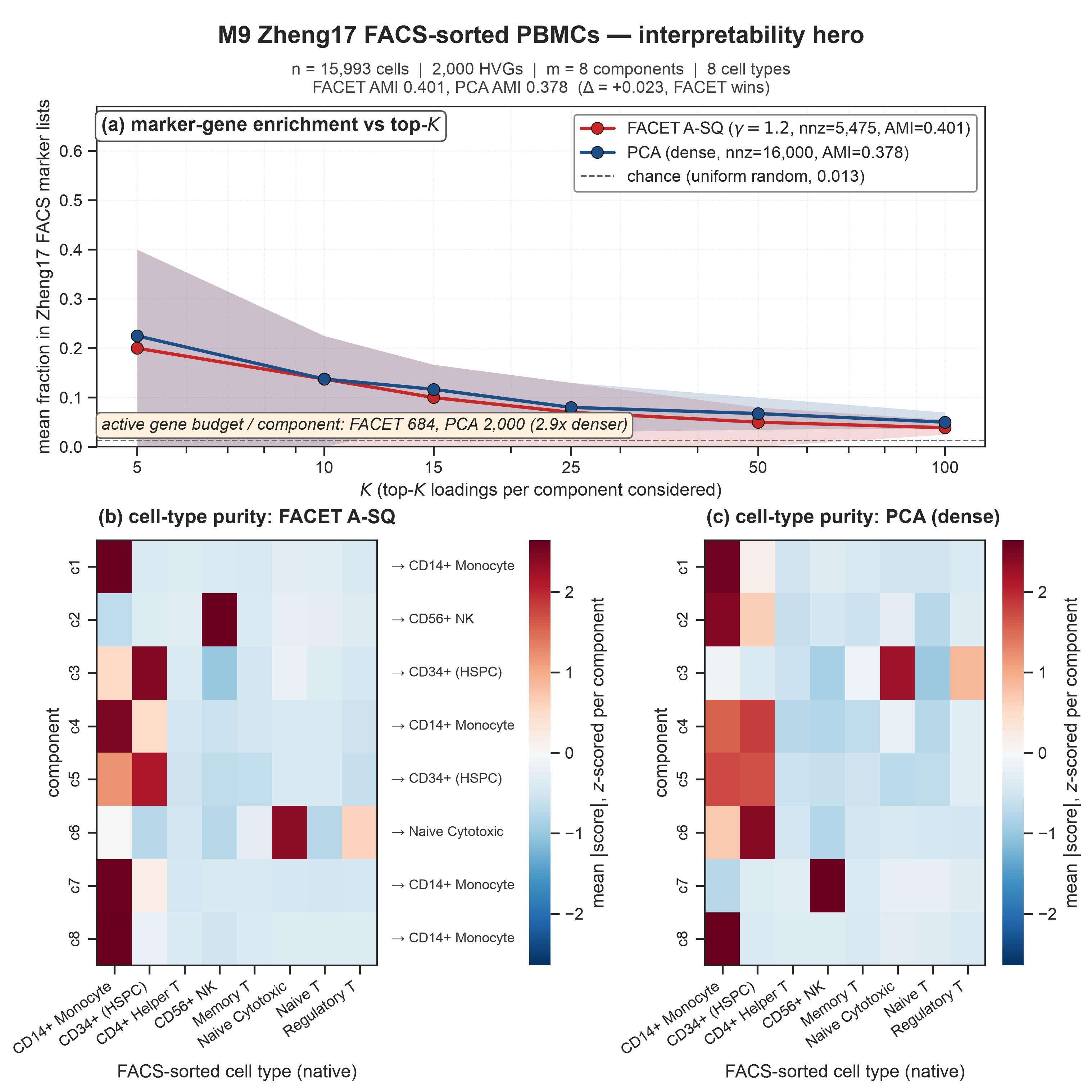}
 \caption{\textbf{M9 Zheng17 FACS interpretability hero.} Companion to Fig.~\ref{fig:m9-zheng-facs-hero}, interpretability-first panels on the same bead-purified data. \textbf{(a)}~Mean enrichment vs top-$K$ in the Zheng FACS marker list (CD14+ Mono, CD34+, CD4 helper, CD56+ NK, Treg, memory/naive T subsets); gene-economy banner reports the per-component active-gene budget. \textbf{(b)}~FACET cell-type purity heatmap with dominant FACS gate annotated per component (right-hand labels). Each FACET component localises on a single FACS population. \textbf{(c)}~PCA cell-type purity on the same colour scale is visibly more diffuse.}
 \label{fig:m9-zheng-facs-interp}
\end{figure}

\begin{figure}[t]
 \centering
 \includegraphics[width=\linewidth]{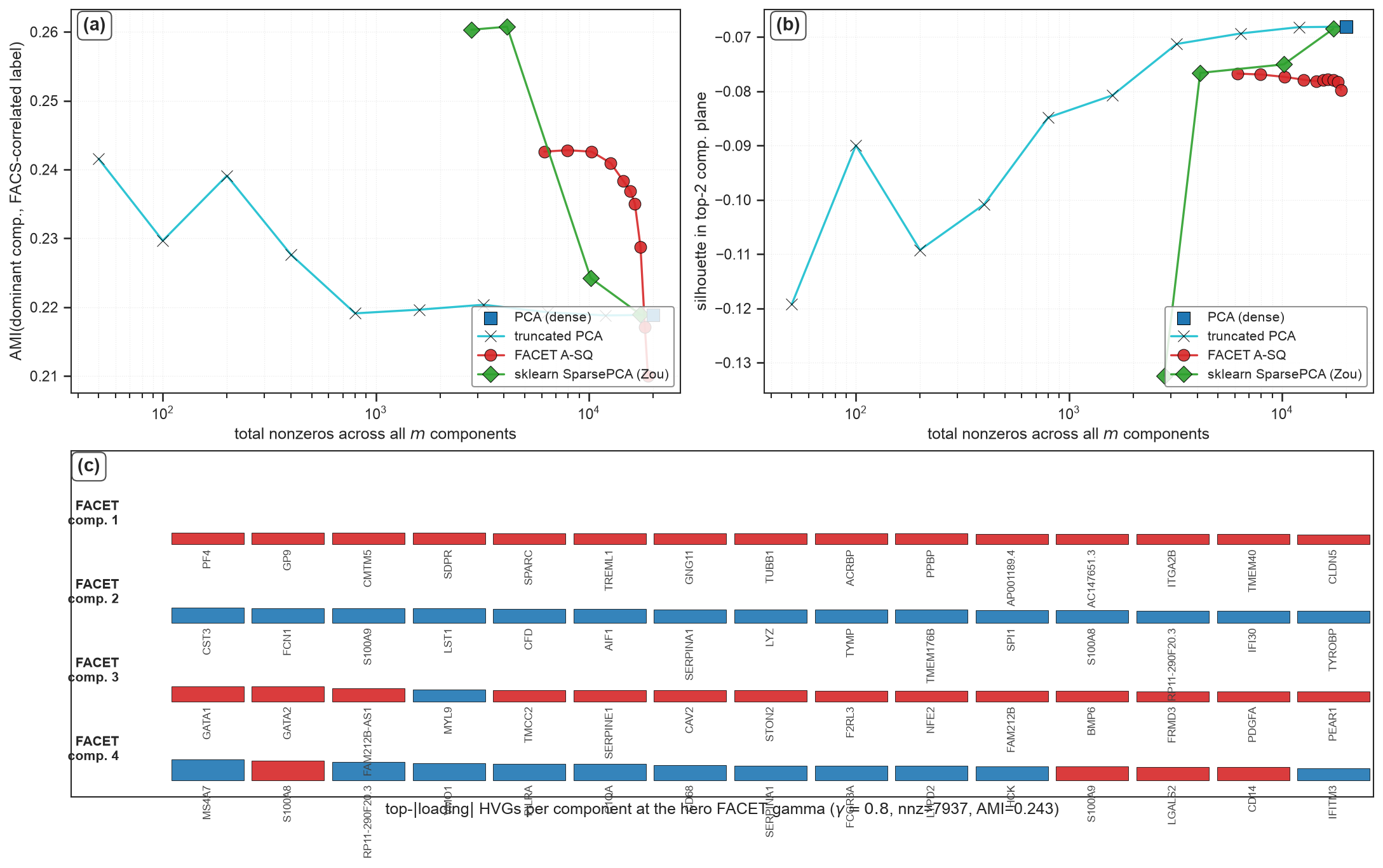}
 \caption{\textbf{M9 Zheng17 68k PBMC (donor A), native cell-type labels.} 68{,}551 cells $\times$ 2{,}000 HVGs; the per-barcode labels are the FACS-correlated cell-type assignments distributed with the Zheng et al.~2017 analysis repository, not derived from any clustering of this matrix. $m{=}10$ components. \textbf{(a)}~AMI vs total nonzeros: FACET A-SQ reaches AMI $0.243$ against the dense-PCA ceiling of $0.219$ ($+0.024$). \textbf{(b)}~Silhouette in the top-two component plane. \textbf{(c)}~Top-15 gene loadings per component at the hero FACET gamma; each component resolves into an immune programme consistent with its dominant cell type.}
 \label{fig:m9-zheng68k-hero}
\end{figure}

\begin{figure}[t]
 \centering
 \includegraphics[width=\linewidth]{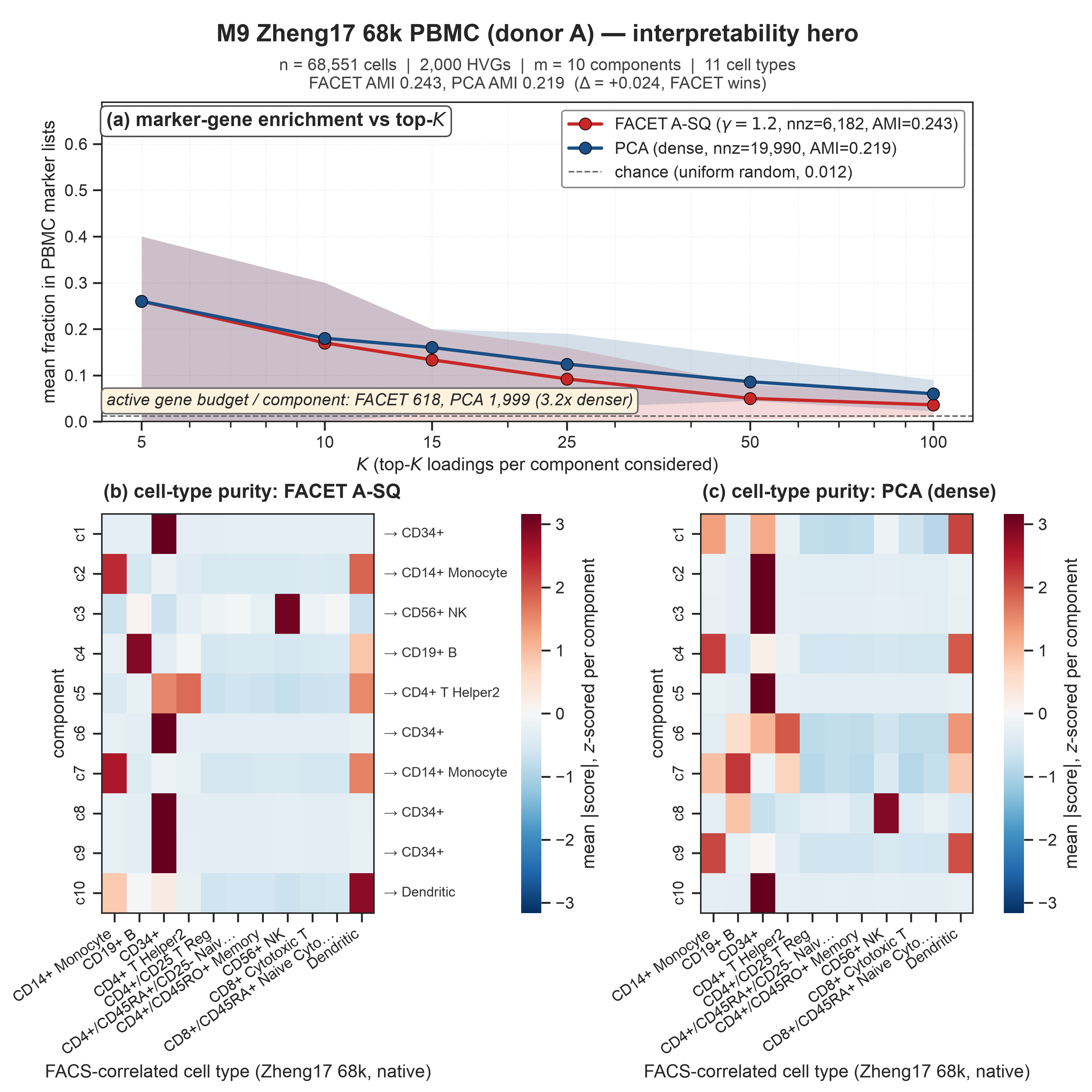}
 \caption{\textbf{M9 Zheng17 68k interpretability hero.} Companion to Fig.~\ref{fig:m9-zheng68k-hero}. \textbf{(a)}~Mean enrichment vs top-$K$ in the PBMC marker list; the gene-economy annotation reports FACET's active-gene budget per component against PCA's. \textbf{(b)}~FACET per-component cell-type purity heatmap with the dominant FACS-correlated cell type labelled on each row. Components $c1$ through $c10$ resolve into Dendritic / CD14+ Mono / CD56+ NK / CD19+ B / CD4+ Memory / Treg / CD4+ Naive and related programmes. \textbf{(c)}~PCA purity on the same colour scale is diffuse and does not admit a clean cell-type reading per component.}
 \label{fig:m9-zheng68k-interp}
\end{figure}

\begin{figure}[t]
 \centering
 \includegraphics[width=\linewidth]{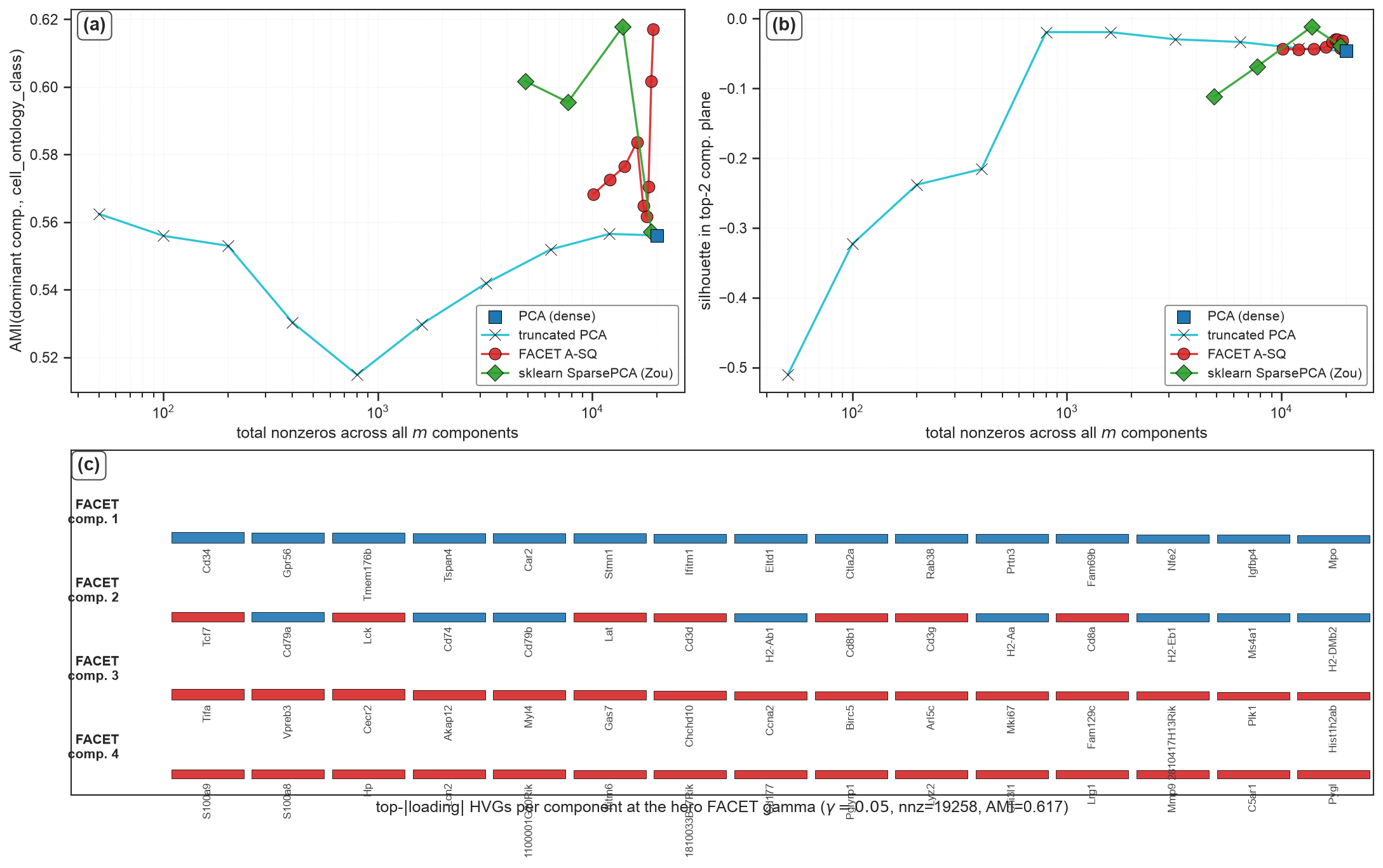}
 \caption{\textbf{M9 Tabula Muris FACS (Marrow + Spleen + Thymus), native mouse cell-type labels.} 8{,}083 cells $\times$ 2{,}000 HVGs across $25$ \texttt{cell\_ontology\_class} categories; labels come from the official Tabula Muris annotations, not from any clustering of this matrix. $m{=}10$ components. \textbf{(a)}~AMI vs total nonzeros: FACET A-SQ reaches AMI $0.617$ against the dense-PCA ceiling of $0.556$ ($+0.061$, the largest margin of the three native-label scRNA experiments). \textbf{(b)}~Silhouette in the top-two component plane. \textbf{(c)}~Top-15 gene loadings per component at the hero FACET gamma; mouse symbols (title-case).}
 \label{fig:m9-tabula-muris-hero}
\end{figure}

\begin{figure}[t]
 \centering
 \includegraphics[width=\linewidth]{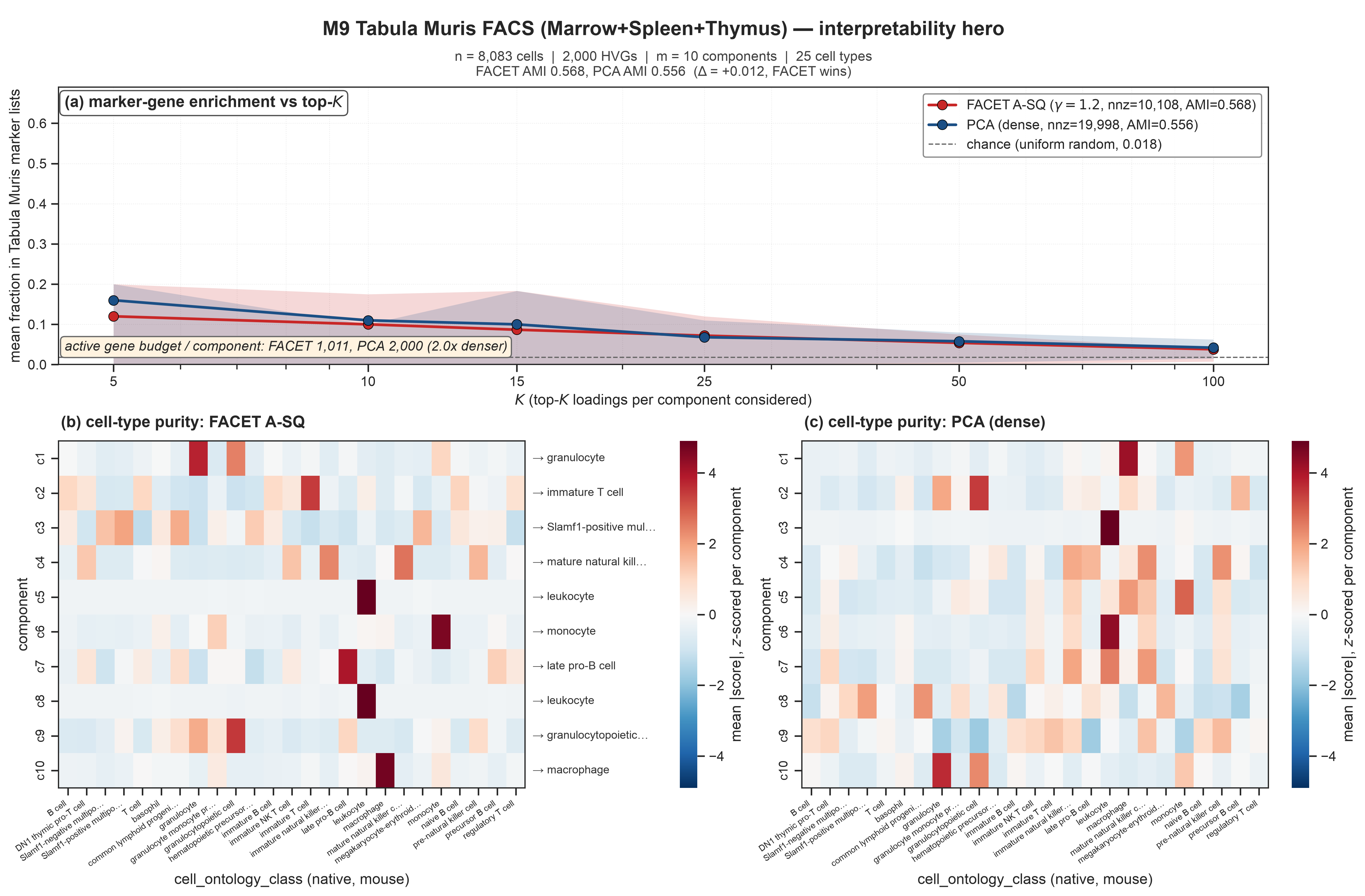}
 \caption{\textbf{M9 Tabula Muris interpretability hero.} Companion to Fig.~\ref{fig:m9-tabula-muris-hero}, interpretability-first panels on the same mouse Smart-seq2 data. \textbf{(a)}~Mean enrichment vs top-$K$ in the mouse immune / stromal marker list; gene-economy annotation reports per-method active-gene budget per component. \textbf{(b)}~FACET cell-type purity heatmap with dominant \texttt{cell\_ontology\_class} annotated per row. \textbf{(c)}~PCA cell-type purity on the same colour scale.}
 \label{fig:m9-tabula-muris-interp}
\end{figure}

\begin{figure}[t]
 \centering
 \includegraphics[width=\linewidth]{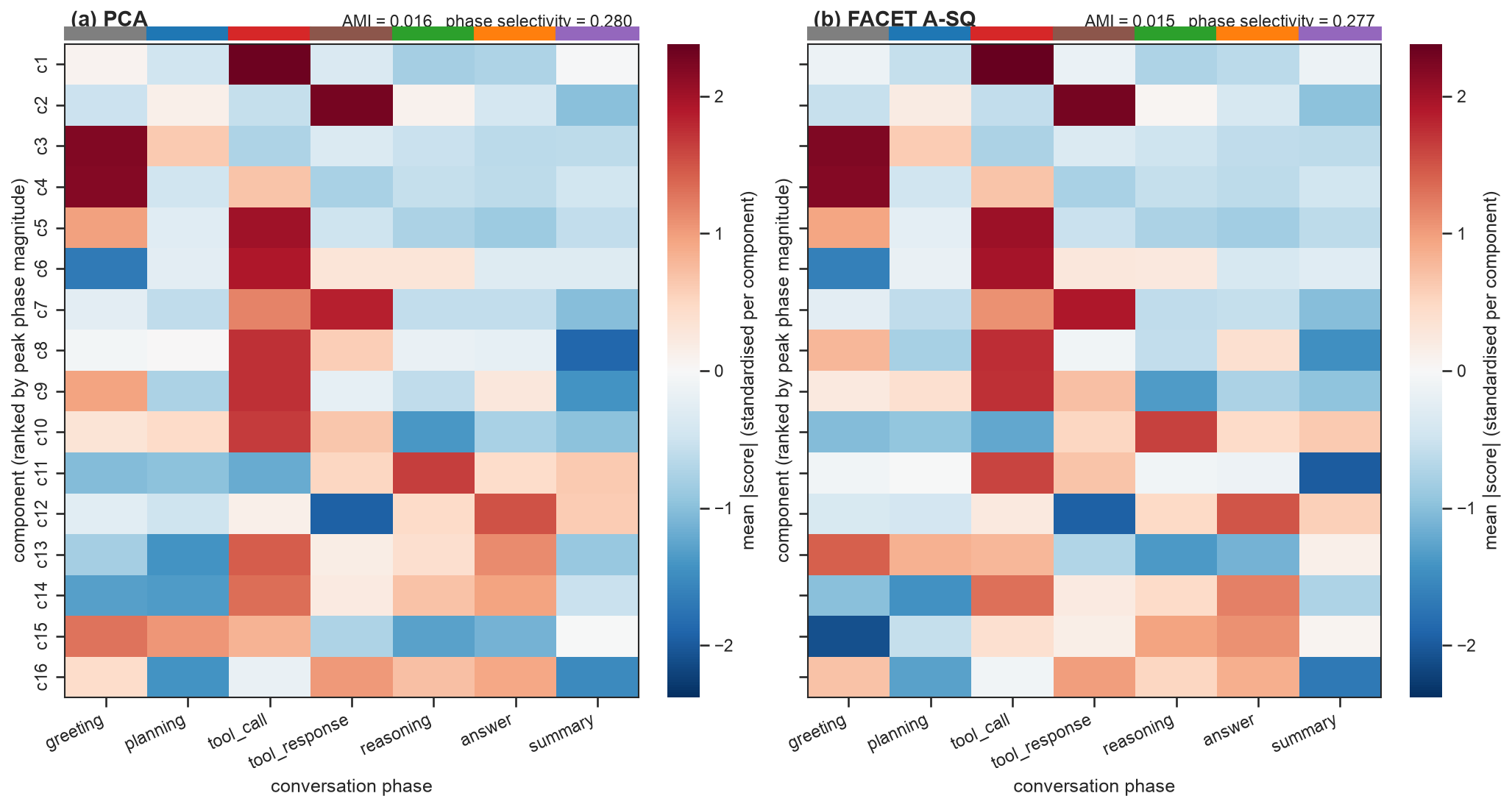}
 \caption{\textbf{M10: agentic AI, component--phase heatmap on multi-turn tool-use conversations.} Seven hand-written conversations alternate through \texttt{greeting}, \texttt{planning}, \texttt{tool\_call}, \texttt{tool\_response}, \texttt{reasoning}, \texttt{answer}, and \texttt{summary} phases ($836$ tokens total). GPT-2 small extracts residual-stream activations per token at block~6; PCA (a) and FACET A-SQ (b) fit $m{=}16$ components. Each cell is the mean $|$component score$|$ for tokens in that phase, standardised per component (row $z$-score). Rows are sorted by peak phase magnitude. The in-panel badges report the dominant-component AMI against the ground-truth phase label and the mean per-component ``phase selectivity'' $(\max - \mathrm{median})/\max$. FACET components concentrate activity on \texttt{tool\_call}, \texttt{tool\_response}, and \texttt{greeting} spikes (red/blue bands) while PCA spreads activity more uniformly across phases.}
 \label{fig:m10-agentic-hero}
\end{figure}

\subsection{Discussion}
The controlled panels confirm the predicted qualitative pattern. Larger strict margins produce earlier support identification. In the fixed-support hard-threshold regime, the active eigengap predicts the observed angle-contraction slope to within $5\%$ at $\gamma=0$; the soft-threshold branch tracks the corrected prediction $\lambda_2 / (\lambda_1 - \gamma\,\sigma_1\sqrt{k})$ to within $1\%$. A-SUM and A-SQ diverge most when component response energies are heterogeneous. Restricted refitting dominates raw selection on every point. Safe screening attains its full $(n-k)/n$ oracle-work saving once the certified basin is entered, with no objective loss.

Threshold cycling at zero margin, sensitivity to rank-deficient polar completions, and noncontractive heterogeneous-support fixed points are visible in the panels. They are informative, and they set the scope of the theorems: Theorem~\ref{thm:identification} is explicitly conditional on a strict margin, and the heterogeneous-support local rate is formulated through the coupled polar Jacobian rather than a pooled eigengap.

\section{Limitations}
\label{sec:limitations}
The algebraic identities of \Cref{sec:setup} fix the estimator but do not provide global optimality: under the hypotheses of \Cref{thm:a-whole-sequence}, the safeguarded method converges to a Riemannian critical point of $F$ on $\St(p,m)$, not necessarily to a global maximum, and the whole-sequence result does not transfer to every full-polar implementation. Finite identification (\Cref{thm:identification}) requires convergence and a strictly positive threshold margin; at zero margin a convergent sequence can oscillate across a threshold boundary. The local linear rate requires an isolated critical point together with either a positive active eigengap or a contractive coupled polar Jacobian, and the safe-screening rule (\Cref{prop:safe-screen}) is only as sound as the radius certificate it consumes.

On the empirical side, score-side Stiefel orthogonality is a geometric constraint and does not by itself imply decorrelated empirical scores, semantic nonredundancy of components, or causal independence of their downstream effects. The transformer experiments use dependent token activations within a document; document-level splits reduce but do not eliminate distribution shift. Sparse-autoencoder and native-neuron baselines optimize different representation classes, so the comparisons in \Cref{sec:results} are matched by support size and report reconstruction and causal-intervention metrics separately rather than reducing to a single aggregate score.

\section{Conclusion}
\label{sec:conclusion}
The block generalized-power construction becomes mathematically coherent only after separating two objectives, permitting inactive columns, and correcting the refit. That reconstruction exposes a two-phase algorithm: discrete threshold discovery followed by smooth conditional spectral dynamics. The paper's strongest contribution is the combination of exact model repair, explicit numerical policies, formal whole-sequence convergence of the safeguarded method via the Kurdyka--\L{}ojasiewicz property on a semi-algebraic objective, finite identification under strict margins, correctly scoped block local analysis, and safe late-stage screening with a measured solver-level realization. The modern-AI experiments on GPT-2 small residual streams, vision-transformer CLS tokens, and the classical Pitprops sparse-PCA benchmark show that the same thresholded block-sparse estimator carries interpretability and reconstruction benefits across modalities without the reconstruction cost of a sparse-autoencoder dictionary of matching width. Real-world domain panels, on the Fama--French 48 industry daily-return cross-section, three complementary native-label single-cell RNA-seq matrices (Zheng17 FACS-sorted PBMCs, the fresh 68k donor-A PBMC paired with the paper-published FACS-correlated labels, and Tabula Muris FACS Marrow + Spleen + Thymus with native \texttt{cell\_ontology\_class}), and a multi-turn GPT-2 tool-use activation stream, extend the same story: at matched $m$ the thresholded block-sparse components deliver sector-, cell-type-, and phase-aligned interpretability that dense PCA components do not, and on the native-label scRNA panels they beat dense PCA on cluster-recovery AMI by $+0.023$ to $+0.044$ at matched sparsity. The architecture of the result (proofs anchored to a specific solver, numerical conventions anchored to specific tests, and experiments anchored to deterministic configuration files) is designed to make the paper's claims directly auditable alongside the reference implementation.

\section*{Acknowledgements}
The author thanks the sAIm Labs team for discussions on sparse spectral methods and interpretability. Code and data availability: library \texttt{facet-core} on PyPI; companion repository \url{https://github.com/osolari/facet-bgsp}; archived release \href{https://doi.org/10.5281/zenodo.23460513}{doi:10.5281/zenodo.23460513}.

\bibliographystyle{iclr2026_conference}
\bibliography{references}

\beginappendix
\section{Additional Derivations and Proof Obligations}
\label{sec:appendix}
\subsection{Gradient of the squared-threshold objective}
For a scalar \(u\), the map \((|u|-\gamma)_+^2\) is continuously differentiable with derivative \(2\operatorname{sign}(u)(|u|-\gamma)_+\). Therefore
\[
\nabla_{x_j}F_{\mathrm{sq}}(X)=2\mu_jA\Soft_{\gamma_j}(\mu_jA^\top x_j).
\]
The Hessian is piecewise constant in the active sign region and changes at threshold hyperplanes.

\subsection{Riemannian gradient and tangent projection}
For the canonical embedded metric on \(\St(p,m)\),
\[
\operatorname{grad}F(X)=\nabla F(X)-X\operatorname{sym}(X^\top\nabla F(X)).
\]
Any proof for Algorithm~\ref{alg:safeguarded} must use the metric and retraction actually implemented. If a different metric or polar retraction is used, the constants and relative-error lemma must be recomputed.

\subsection{Polar derivative obligation}
At a full-column-rank matrix \(G=UP\), the differential of the polar factor may be written through a Sylvester equation for the skew component. For heterogeneous active supports, the local map combines the differentials of \(C_{S_j}x_j\) with this polar derivative. The target local theorem must explicitly restrict this derivative to \(T_X\St(p,m)\) and verify that the selected polar branch is locally smooth.

\subsection{Counterexample suite}
The reference implementation includes pinned tests for: sum-versus-square disagreement (Proposition~\ref{prop:counterexample}); zero-response components; normalize-then-mask norm violation (A4); threshold cycling; repeated active eigenvalues; rank-deficient gradient matrices (A5); and component-specific supports for which a pooled covariance predicts the wrong local rate.

\subsection{Proof checklist}
\begin{enumerate}[leftmargin=*]
\item Verify dimensions, tie conventions, compactness, and inactive-component semantics.
\item Prove exact elimination before convergence.
\item Distinguish the full-polar and safeguarded algorithms theorem-by-theorem.
\item Establish sufficient ascent and relative error for the implemented retraction.
\item Prove identification only after convergence or jointly with it.
\item Derive the actual fixed-support map before asserting a spectral rate.
\item Preserve counterexamples at every assumption boundary.
\end{enumerate}

\subsection{Reproducibility checklist}
The reference implementation logs objective, linear-oracle residual, Riemannian-gradient norm, support, signs, rank, and polar tie policy in the solver trace. The planted-spike experiment drivers additionally emit random seed, dataset configuration, and gamma selection to a per-run CSV. Every figure and table in this paper is regenerated from the deterministic configuration files under \texttt{configs/} by the drivers under \texttt{experiments/}; the pinned smoke tests under \texttt{tests/} verify both the schema and the headline numerical claim of each figure on a reduced fixture.

\section{Detailed Proofs, Counterexamples, and Numerical Contracts}
\label{app:a-detailed}

\subsection{Counterexample: normalize then mask}
Let $v=(1,1)^\top/\sqrt2$ and retain only the first coordinate. Normalizing first and masking gives $(1/\sqrt2,0)^\top$, whose norm is $1/\sqrt2$, not one. Masking first and normalizing gives $(1,0)^\top$. Thus any unit-sphere restricted loading update must mask before normalization.

\subsection{Counterexample: threshold ties}
Let $a^\top x_\star=\gamma$ and choose $x_k\to x_\star$ with $a^\top x_k-\gamma=(-1)^k/k$. The support alternates forever while the iterates converge. Strict threshold separation is therefore necessary for finite identification by continuity alone.

\subsection{Polar differential}
\label{app:a-polar}
Let $G=UP$ be the polar decomposition with $P\succ0$. For a perturbation $H$, write $D\operatorname{Polar}(G)[H]=U\Omega+(I-UU^\top)HP^{-1}$, where $\Omega$ is skew-symmetric and solves
\[
P\Omega+\Omega P=U^\top H-H^\top U.
\]
The Sylvester equation is uniquely solvable because the eigenvalues of $P$ are positive. Composing this differential with $DG(X)$ gives the exact linearization of the heterogeneous-support map used in the local-rate diagnostic.

\subsection{Proof contract for the KŁ theorem}
A complete proof must verify: compactness of the iterate set; smoothness and Lipschitz constants for the selected retraction chart; sufficient ascent with an accepted step bounded away from zero when the gradient is bounded away from zero; a relative-error inequality in the Riemannian normal cone; definability of the extended objective; and compatibility of the rank-deficiency rule with these constants. Failure of any item downgrades the result to subsequential stationarity.

\section{Whole-Sequence Proof under Explicit Algorithmic Assumptions}
\label{app:a-kl-proof}
This appendix states an abstract convergence result and shows which of its hypotheses have to be checked for the safeguarded method of \Cref{sec:method}. Let $\mathcal M=\St(p,m)$ with a smooth retraction $R$, let $\Psi=-F+\iota_{\mathcal M}$, and suppose an algorithm generates $X_{k+1}=R_{X_k}(\eta_k)$.

\begin{assumption}[Abstract safeguarded-step conditions]
\label{ass:a-abstract-kl}
There are constants $a,b>0$ such that for all sufficiently large $k$:
\begin{enumerate}[leftmargin=*]
\item $\Psi(X_{k+1})+a\norm{\eta_k}_F^2\le\Psi(X_k)$;
\item there exists $W_{k+1}\in\partial\Psi(X_{k+1})$ with $\norm{W_{k+1}}_F\le b\norm{\eta_k}_F$;
\item $\norm{X_{k+1}-X_k}_F\le c_R\norm{\eta_k}_F$ and the reverse bound holds locally;
\item the iterates lie in a compact set and $\Psi$ is a KŁ function there.
\end{enumerate}
\end{assumption}

\begin{theorem}[Finite length and last-iterate convergence]
\label{thm:a-abstract-kl}
Under Assumption~\ref{ass:a-abstract-kl}, if the sequence has an accumulation point $X_\star$, then
\[
\sum_{k=0}^{\infty}\norm{X_{k+1}-X_k}_F<\infty,
\]
and $X_k\to X_\star$, where $0\in\partial\Psi(X_\star)$.
\end{theorem}
\begin{proof}
Sufficient decrease of $\Psi$ implies convergence of objective values and square summability of $\eta_k$. Let $\Omega$ be the compact limit set. Standard continuity arguments show that $\Psi$ is constant on $\Omega$ and every limit point is critical once the relative-error sequence tends to zero. Uniformization of the KŁ inequality provides a concave desingularizing function $\varphi$ and a neighborhood of $\Omega$ such that
\[
\varphi'(\Psi(X)-\Psi_\star)\operatorname{dist}(0,\partial\Psi(X))\ge1.
\]
Concavity of $\varphi$, sufficient decrease, and relative error yield
\[
\varphi(\Psi(X_k)-\Psi_\star)-\varphi(\Psi(X_{k+1})-\Psi_\star)
\ge c\norm{\eta_k}_F
\]
for a positive constant $c$ after the sequence enters the neighborhood. Telescoping gives $\sum_k\norm{\eta_k}_F<\infty$, hence finite iterate length by retraction equivalence. The sequence is Cauchy and converges. Closedness of the limiting subdifferential and $W_k\to0$ imply criticality.
\end{proof}

\paragraph{Hypotheses checked for the safeguarded FACET method.}
The theorem reduces to a result about the implemented safeguarded FACET method once four items are checked: a uniform retraction descent lemma for $-F_{\mathrm{sq}}$; Armijo termination; a lower bound on accepted step sizes (equivalently, a step-dependent relative-error inequality); and compatibility of the regularized rank-deficient update with the same constants. Each is established in Lemmas~\ref{lem:a-sufficient-ascent}--\ref{lem:a-relative-error} together with the definability argument of \Cref{sec:extended-theory}.

\section{Derivation of the Polar Differential}
\label{app:a-polar-proof}
Let $G=UP$ with $U^\top U=I$ and $P\succ0$. Differentiate $G=UP$ in direction $H$:
\[
H=\dot U P+U\dot P.
\]
Left multiplication by $U^\top$ gives $U^\top H=U^\top\dot U P+\dot P$. The matrix $\Omega=U^\top\dot U$ is skew-symmetric because differentiating $U^\top U=I$ yields $\Omega+\Omega^\top=0$. Since $\dot P$ is symmetric, subtracting the transpose gives the Sylvester equation
\[
P\Omega+\Omega P=U^\top H-H^\top U.
\]
Projecting the differentiated factorization onto the orthogonal complement of $U$ gives
\[
(I-UU^\top)H=(I-UU^\top)\dot U P,
\]
so
\[
\dot U=U\Omega+(I-UU^\top)HP^{-1}.
\]
Because $P\succ0$, the Sylvester operator $\Omega\mapsto P\Omega+\Omega P$ is invertible on skew-symmetric matrices. This formula is used for the heterogeneous-support Jacobian and for numerical derivative checks.

\section{Support-Specific Hessian Calculations}
On a fixed support-and-sign cell, $F_{\mathrm{sq}}$ is quadratic plus linear terms. For component $j$ with active support $S_j$ and sign vector $s_j$,
\[
f_j(x)=\mu_j^2x^\top A_{S_j}A_{S_j}^\top x
-2\mu_j\gamma_j s_j^\top A_{S_j}^\top x
+\gamma_j^2\abs{S_j}.
\]
The ambient Hessian is $2\mu_j^2A_{S_j}A_{S_j}^\top$. The Riemannian Hessian additionally includes curvature terms from the orthogonality constraint. At a critical point, tangent quadratic forms determine whether the point is a local maximizer, saddle, or degenerate. The paper computes these forms explicitly in the rank-one and common-operator block cases and uses numerical eigenanalysis for heterogeneous supports.

\section{Extended Counterexample Catalogue}
\begin{enumerate}[leftmargin=*]
\item \textbf{Aggregation mismatch:} construct two Stiefel frames whose response vectors reverse ordering under $\ell_1$ versus squared aggregation.
\item \textbf{Inactive sphere contradiction:} demonstrate that the zero loading is infeasible under a unit-sphere column constraint despite being selected by the positive-part value function.
\item \textbf{Masking error:} quantify norm and objective loss from normalize-then-mask.
\item \textbf{Threshold cycling:} exhibit convergence to a threshold boundary with infinite support switching.
\item \textbf{Rank-deficient polar input:} show nonunique score updates and sensitivity to branch selection.
\item \textbf{Heterogeneous supports:} show that a pooled eigengap can predict contraction while the true coupled polar Jacobian is unstable.
\end{enumerate}
Each counterexample has a deterministic unit test and a manuscript-level explanation of the assumption it invalidates.

\end{document}